\documentclass{article}

\usepackage[preprint, nonatbib]{neurips_2024}

\usepackage[utf8]{inputenc} 
\usepackage[T1]{fontenc}    
\usepackage{hyperref}       
\usepackage{url}            
\usepackage{booktabs}       
\usepackage{amsfonts}       
\usepackage{nicefrac}       
\usepackage{microtype}      
\usepackage{xcolor}         

\usepackage{graphicx}
\usepackage{bbm}
\usepackage{subcaption}
\usepackage{booktabs} 
\usepackage{hyperref}
\usepackage[numbers]{natbib}
\usepackage{algorithm} 
\usepackage{algpseudocode} 

\usepackage{amsmath}
\usepackage{amssymb}
\usepackage{mathtools}
\usepackage{amsthm}

\usepackage{comment} 
\usepackage[capitalize,noabbrev]{cleveref}

\theoremstyle{plain}
\newtheorem{theorem}{Theorem}
\newtheorem{lemma}{Lemma}

\theoremstyle{definition}
\newtheorem{example}{Example}
\newtheorem{definition}{Definition}
\newtheorem{assumption}{Assumption}
\theoremstyle{remark}

\usepackage[textsize=tiny]{todonotes}
\def\Ds{\mathcal{D}_{s}}
\def\ind{\mathbbm{1}}
\def\Dt{\mathcal{D}_{t}}

\def\R{\mathbb{R}}

\def\cX{\mathcal{X}}
\def\cC{\mathcal{C}}
\def\cY{\mathcal{Y}}

\def\cZ{\mathcal{Z}}
\def\cQ{\mathcal{Q}}
\def\cq{q}
\def\cH{\mathcal{H}}

\def\n{\mathcal{N}}

\def\cD{\mathcal{D}}
\def\dist{\Delta}
\def\dd{\partial}
\def\Lin{Lin}
\def\nERM{\textsc{feature\_validate}}
\def\ualg{\textsc{presrv\_contract\_nn}}
\def\nLabel{\textsc{direct\_generalize\_nn}}
\def\bayes{\Gamma}
\def\marginx{\rho}
\def\marginy{\Delta}
\def\cor{\textnormal{Cor}}
\def\proj{\textnormal{Proj}}
\DeclareMathOperator*{\argmax}{arg\,max}
\DeclareMathOperator*{\argmin}{arg\,min}
\newcommand{\supp}{\textnormal{supp}}
\def\bayes{g}

\def\coverconst{\Lambda}
\def\mv{\alpha}     

\def\Assumption{Statistical IRM Assumption}

\def\Phis{\mathcal{S}(\Phi)}  
\def\Phic{\mathcal{C}(\Phi)} 
\def\Phiu{\mathcal{U}(\Phi)} 
\def\Phirealize{\Phi^*} 

\title{Beyond Discrepancy: A Closer Look at the Theory of Distribution Shift}

\author{
Robi Bhattacharjee  \\
University of Tübingen and Tübingen AI Center \\
 \texttt{robi.bhattacharjee@wsii.uni-tuebingen.de} \\
\And
Nick Rittler  \\
University of California- San Diego\\
\texttt{nrittler@ucsd.edu} \\
\And
Kamalika Chaudhuri \\
University of California - San Diego\\
\texttt{kamalika@cs.ucsd.edu} \\
}

\begin{document}

\maketitle

\begin{abstract}
Many machine learning models appear to deploy effortlessly under distribution shift, and perform well on a target distribution that is considerably different from the training distribution. Yet, learning theory of distribution shift bounds performance on the target distribution as a function of the discrepancy between the source and target, rarely guaranteeing high target accuracy. Motivated by this gap, this work takes a closer look at the theory of distribution shift for a classifier from a source to a target distribution. Instead of relying on the discrepancy, we adopt an Invariant-Risk-Minimization (IRM)-like assumption connecting the distributions, and characterize conditions under which data from a source distribution is sufficient for accurate classification of the target. When these conditions are not met, we show when only unlabeled data from the target is sufficient, and when labeled target data is needed. In all cases, we provide rigorous theoretical guarantees in the large sample regime.
\end{abstract}


\section{Introduction}

Classical learning theory operates within the statistical learning framework, in which the training and testing datasets are assumed to be drawn from the same distribution \cite{valiant1984}. However, this assumption is rarely met in practice, where models often succeed in ever-changing real world environments rarely matching the precise conditions of their training data. This motivates the problem of distribution shift, in which a learner trains on a source distribution, with the goal of generalizing well over a distinct target distribution.

Thus far, the theory of distribution shift has consistently taken a worst-case approach, typically bounding generalization error in terms of some notion of discrepancy between the source and target distributions \citep{bendavid2006, bendavid2010, mansour2012}. In cases where the source and target distributions are completely unrelated, or the source provides little information about the decision boundary of the target, discrepancy-based analyses correctly capture the difficulty of generalization. However, in practice, many large models appear to generalize effortlessly to target distributions with non-zero discrepancy. 

Motivated by this gap, we take a closer look at the theory of distribution shift. In our setting, we consider a source distribution $\Ds$ and a target distribution $\Dt$, with the goal of building an accurate classifier over $\Dt$, primarily via training samples from $\Ds$. To accomplish this, we first select a feature map, $\hat{\phi} \in \Phi$, under which the source and target distributions are similar. To make predictions, we then use $k_n$-nearest neighbors ($k_n$-NN) inside feature space over data sampled from $\Ds$.  

Instead of a worst-case, discrepancy-based approach, we study generalization under an Invariant Risk Minimization (IRM)-like assumption which we term the ``Statistical IRM Assumption''. IRM assumes on the existence of a feature map $\psi^*$ and a classifier (over feature space) $h^*$ so that their composition $h^* \circ \psi^*$ achieves optimal accuracy over both source and target distributions \cite{arjovsky2020}. We adapt this assumption to the nearest neighbors setting, and replace the existence of $h^*$ with the assumption that the some feature map $\phi^* \in \Phi$ maps points from the target $\Dt$ close to those from the source $\Ds$ while retaining information sufficient for optimal prediction. This property allows us to leverage the fact that nearest neighbors enjoys strong generalization properties within the support of its training distribution. 

One might hope that such a condition is sufficient for generalization from source data alone. Unfortunately, the existence of a suitable feature map does not imply its identifiability - there may be many poor feature maps in $\Phi$ that appear suitable when only source data are available. We show (Theorem \ref{theorem:type-1-lower-bound}) that to guarantee generalization to the target using only source data, the source must be rich enough so that this cannot happen, i.e. that \textit{all} maps that lead to optimal classification over the source distribution appropriately unify the source and target. We further exhibit a learning rule which leads to provable generalization to the target under this additional condition (Theorem \ref{thm:setting_2_upper_bound}). 

\begin{figure*}
\centering
  \begin{subfigure}[t]{.32\linewidth}
    \includegraphics[width=\linewidth]{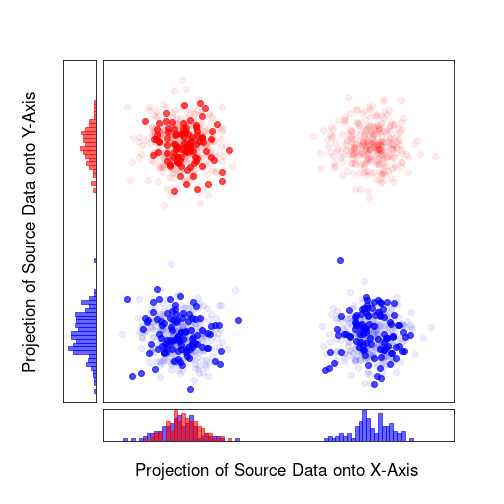}
    \caption{Pure source data is sufficient, as $\phi_x$ significantly reduces accuracy over the source distribution.}
  \end{subfigure}
    \hfill
   \begin{subfigure}[t]{.32\linewidth}
    \includegraphics[width=\linewidth]{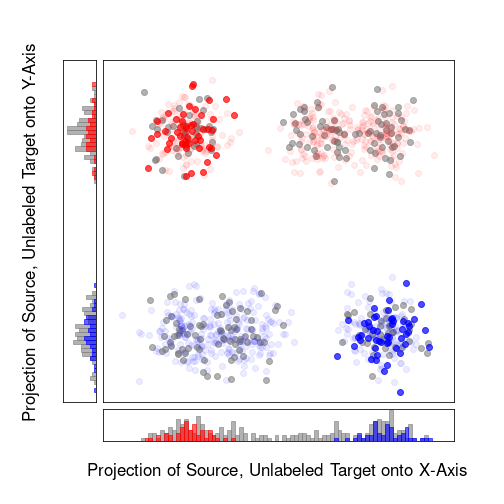}
    \caption{Unlabeled target data coupled with labeled source data is sufficient -- $\phi_x$ can be eliminated because it does not map target data close to source data in the implied feature space.}
  \end{subfigure}
  	\hfill
  \begin{subfigure}[t]{.32\linewidth}
    \includegraphics[width=\linewidth]{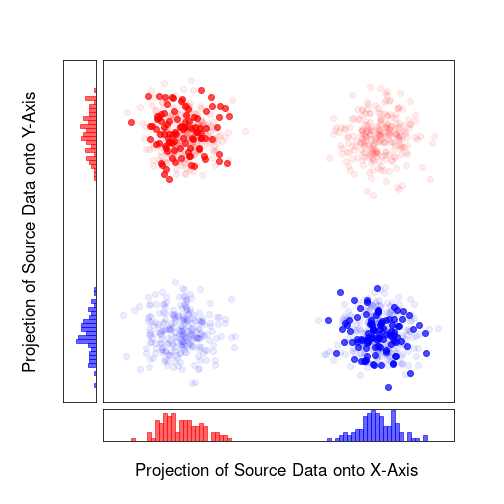}
    \caption{Labeled target data is needed to determine which projection is correct -- without it both projections appear completely symmetric, but with competing notions of how to label in feature space. }
  \end{subfigure}  
  \caption{Examples of similar distribution shift problems with different data demands. Faded data points are sampled from the target distribution, while the bold points are selected from the source. In all cases, we wish to generalize to the target via the selection a feature map from $\Phi = \{\phi_x, \phi_y\}$, with $\phi_x$ and  $\phi_y$ denoting projection onto the $x$ and $y$ axis, respectively.} \label{fig:page_2_fig}
\end{figure*}

We next consider the case where the learner has access to unlabeled target data in addition to labeled source data. Here, the target data provides crucial new information about which feature maps transform target data close to source data in feature space. We find that it is necessary and sufficient (Theorems \ref{thm:lower_bound_setting_unlabeled} and \ref{thm:unlabel_upper_bound}) that all maps which both lead to optimal classification over the source, and map target data close to source data, further appropriately unify the source and target classification tasks.

When generalization is not possible with the addition of unlabeled target data, some labeled target data is needed. In this setting, the goal is to minimize the amount of labeled target data used -- if large amounts of labeled target data are obtainable, we could simply use standard learning algorithms directly on the target data. We introduce a complexity measure on the embedding class, $\Phi$, which we term the \textit{distance dimension}, and use it to provide an upper bound on the amount target data needed for generalization. In particular, we show that the natural procedure of minimizing the empirical risk (over $\phi \in \Phi$) on the target distribution of the source data-trained nearest neighbor classifier meets this upper bound. 

\subsection{An Illustrative Example}\label{sec:intro_example}
Figure \ref{fig:page_2_fig} illustrates three learning problems in which we seek to generalize from the bold source data to the faded target data. In each case, the set of possible features maps is $\Phi = \{\phi_x, \phi_y\}$, the projections of onto the $x$ and $y$-axis, respectively. Here, the Statistical IRM Assumption manifests itself in the following way: the learner knows that perfect classification can be performed on the source and the target through the intermediate projection onto either the $x$ or $y$-axis. If the correct projection can be identified, a classifier generalizing to the target can be built by composing the correct projection with a classifier that accurately classifies source data in feature space.

The possibility of generalizing directly to the target is illustrated by Figure \ref{fig:page_2_fig}(a). In this case, using source data alone, it can be deduced that $\phi_x$ is not suitable, given that projection under $\phi_x$ significantly reduces accuracy over the source distribution. Thus, a classifier can be constructed through composition with $\phi_y$ that allows for generalization to the target.

By contrast, in panel (b), we see that both $\phi_x$ and $\phi_y$ admit good classification over the source distribution. However, note that only $\phi_y$ leads to good generalization over the target distribution, and that there is no way to pin down which embedding should be used with source data alone. That said, given access to unlabeled target data, $\phi_x$ can be eliminated from contention -- it fails to uniformly map target data close to source data in feature space, another condition for correctly relating the source and target. 

In panel (c), we see an instance in which no amount of source data and unlabeled target data will allow the learner to distinguish a winner between the two possible feature maps. In this case, labeled target data is needed.  However, note that only a relatively small amount of labeled target data will be needed -- all that is required is enough points to validate that a source-trained classifier arising from first projecting onto the $x$-axis has inferior performance to an analogous classifier where data are projected to the $y$-axis. 

\subsection{Guarantees Beyond Discrepancy}
The examples of Figure \ref{fig:page_2_fig} also serve to showcase the potential for generalization guarantees in scenarios where worst-case analyses indicate that generalization to the target should be hard. 

There are a few veins of the discrepancy literature \cite{hanneke2020}. One prominent vein considers bounding generalization error in terms of divergence measures between the source and target \citep{bendavid2006, bendavid2010, mansour2012}. Another considers density ratios between target and source \cite{quionero2009, sugiyama2012}. In each case, the idea is that the degradation of prediction quality on the target will be small when the source and target distributions are not ``too far'' from each other. 

Consider again the examples of Figure \ref{fig:page_2_fig}. A density ratio analysis indicates that generalization to the target in (a) is impossible from pure source data, and expensive in a transfer learning setting (c), as the source has no mass in large chunks of the support of the target. Divergence measures paint a similar picture. 
Thus, our assumption allows us to consider the possibility of cheap generalization to targets which may have a completely different support from the source in the original data space, but are related in some deeper manner. In such scenarios, discrepancy-based analyses may often be overly pessimistic.

\section{Related Work}


As alluded to above, the theory literature has primarily studied distribution shift through the lens of discrepancy \cite{bendavid2006, bendavid2010, quionero2009, mansour2012, sugiyama2012, nguyen2022, wang2023}. In the transfer learning literature, in which one considers the possibility of updating a model trained on the source distribution with a relatively small amount of target data, divergence-based analyses have also been prevalent \cite{sugiyama2007, mohri2012, cortes2019, mansour2023}. A notable line of work attains strong guarantees in certain cases where the divergence between source and target distributions is large, but the decision boundary on the source and target are similar by honing in on the information about the decision boundary contained in the source distribution \cite{hanneke2020, hanneke2023}.

Much of the attention towards the selection of feature representations has been devoted to problem of ``domain generalization'', wherein the learner tries to generalize to a large of set testing environments using samples from a smaller set of source environments, which provide training data \cite{muandet2013, zhao2017}. As mentioned above, the IRM literature hinges on the existence of a feature map, $\psi^*$ and a classifier $h^*$ whose composition achieves optimal accuracy over both source and target distributions \cite{arjovsky2020}. Another line of work considers a different assumption, namely the existence of some suitable feature map $\psi^*$ for which the conditional distributions of transformed features $\psi^*(x)$ given a label are shared across all environments \cite{chen2021, sun2015, li2018}. 

An important part of this work considers the case where some relatively small amount of labeled target data is available to the learner, and can be exploited in the determination of a suitable feature map. In the theory literature, this setting is most closely explored by the work on `'few-shot representation learning'' \cite{du2021, maurer2016}, where the goal is to use data on a set of source tasks to learn a low dimensional representation that connects tasks together, allowing for generalization to a related target task without too many extra samples.

In considering the case where the learner has access to unlabeled target samples, we enter the `'unsupervised domain adaptation'' setting. Here, one often uses unlabeled target data to find some feature space under which source and target supports align \cite{glorot2011, liu2022}. The literature has shown that unlabeled data has provable utility in certain common settings, e.g. under covariate shift \cite{bendavid2012, huang2006}.  Unsupervised domain adaptation has also been studied through the lens of discrepancy \cite{nguyen2022, wang2023}.


\section{Preliminaries}\label{sec:prelims}

Let the instance space $(\mathcal{X}, d_\cX)$ be a compact metric space, and $\cY$ be a finite label set.  A data distribution $\cD = (\mu, \eta)$ over $\cX \times \cY$ is defined by a Borel measure $\mu$ over $\cX$, and a conditional probability function $\eta(y|x) := \Pr_{(X, Y) \sim \cD}[Y=y \mid X= x]$. 

We assume our distributions satisfy some measure-theoretic regularity conditions. In particular, we assume our Borel measures are open measures, and that the Lebesgue Differentiation Theorem always holds. See Appendix \ref{app:measure} for details to this end.


For a classifier $h: \cX \to \cY$,  we define its \textbf{risk} $R(h, \cD)$ over $\cD$ as the probability it misclassifies, i.e. we define $R(h, \cD) := \Pr_{(X, Y) \sim \cD}[h(X) \neq Y]$. The classifier with the lowest possible risk is called the \textbf{Bayes optimal classifier}, defined as $\bayes_\cD(x) = \argmax_{y \in \cY}\eta(y|x)$. 

\subsection{Problem Statement and Goal}
In this work, we are interested in the problem of distribution shift, in which the goal is to build a classifier with low risk over a target distribution $\Dt = (\mu_t, \eta_t)$, primarily using data from a source distribution, $\Ds = (\mu_s, \eta_s)$. We denote the Bayes risk on source and target via $R_s^*$ and $R_t^*$. 

The challenge in this setting is that $\mu_s$ and $\mu_t$ can put mass in drastically different regions in $\cX$ making direct generalization from the source distribution to the target distribution difficult or impossible in the worst case.

\subsection{Feature Maps}\label{sec:feature_maps_in_main_body}
We consider classification after first applying a transformation into a feature space $(\cZ, d_{\cZ})$, also a compact metric space. 

We assume we are given $\Phi$, a class of feature maps $\phi: \cX \to \cZ$. Here, each $\phi \in \Phi$ represents a potential feature map under which the source and target distributions could plausibly be connected. Let $d_\phi$ denote the distance metric induced on $\cX$ by $\phi$, i.e. $d_\phi(x, x') = d_\cZ\left(\phi(x),\phi(x')\right)$. We assume all $\phi \in \Phi$ are continuous, and $\Phi$ is compact with respect to the supremum distance metric. We also include further technical assumptions on $\Phi$ in Appendix \ref{app:phi_assumptions_new}.

Note that the following important examples of feature map collections which meet these regularity assumptions when the domain is a compact subset of $\mathbb{R}^D$.
\begin{example}\label{def:cor}
Let $\cor_{D, K}$ denote the set of all projections from $\R^D \to \R^K$ onto a set of $K$ coordinates. Formally, we may write $\cor_{D, K} = \{\phi_J: J \subset [D], \ |J| = K\}$,
where for each $J \subseteq [D]$ with $J = \{j_1, \dots, j_k\}$, we let $\phi_J(x) = (x_{j_1}, x_{j_2}, \dots, x_{j_k})$.
\end{example}

\begin{example}\label{def:proj}
Let $\proj_{D, K}$ denote the set of all linear maps corresponding to matrices in $\R^{D \times K}$ with each entry contained in $[-1, 1]$.
\end{example}

For any data distribution $\cD = (\mu, \eta)$ over $\cX \times \cY$, we denote via $\cD^\phi$ the distribution defined via $(\phi(X), Y)$ where $(X, Y) \sim \cD$, often writing $\cD^\phi = (\mu^\phi, \eta^\phi)$, where $\mu^\phi$ and $\eta^\phi$ are the induced marginal and conditional distributions of $\cD^\phi$. We assume that the induced marginals are also open measures. Measure-theoretic details of induced distributions are discussed in Appendix \ref{app:Lebesgue}.

\subsection{Nearest Neighbors}
We let $\n_{S}: \cX \to \cY$ denote the $k_n$-nearest neighbor classifier arising from an i.i.d. sample $S \sim \cD^n$ and a metric over the instances, where ties are broken arbitrarily. It is well known that under mild regularity conditions, $k_n/n \to 0$ and $k_n \to \infty$ imply that $k_n$-nearest neighbors will converge to the Bayes optimal classifier \citep{chaudhuri14}. Motivated by technical concerns, we will make the slightly stronger assumption that $k_n/n \to 0$ and $k_n/\log(n) \to \infty$. 

Because we consider classification in feature space, we will often consider the composition of $k_n$-NN with maps $\phi \in \Phi$. To this end, we let $\n_{S}^\phi: \cX \to \cY$ denote the map defined by 
\begin{equation*}
\n_S^\phi(x) = \n_{\{(\phi(x), y): (x,y) \in S\}} \big( \phi(x) \big).
\end{equation*}

%
%

\subsection{Margin Conditions}
Finally, we restrict our attantions data distributions in which Bayes-optimal classification is clearly non-ambiguous, and regions in which Bayes-optimal predictions differ are separated by a margin. We formalize this as follows.

\begin{definition}\label{defn:well-separated}
A data distribution $\cD$ over $\cX \times \cY$ is $(\marginx, \marginy)$-$\textbf{separated}$ if there exist $\marginx, \marginy > 0$, and disjoint sets $\{\mu^y: y \in \cY\}$, so that the following hold:
\begin{enumerate}
	\item The sets cover the support: $\supp(\mu) = \cup_{y \in \cY} \mu^y$.
	\vspace{-1mm}
	\item On the set where $y$ is the Bayes-optimal decision, no other label has similar conditional probability: If $y \neq y'$, then $\forall x \in \mu^y$, $\eta(y|x) > \eta(y'|x) + \marginy$. 
	\vspace{-1.5mm}
	\item These sets themselves are separated by a margin: $\min_{y \neq y'} d(\mu^y, \mu^{y'}) = \marginx$. 
\end{enumerate}
\end{definition}
When $\cD$ is $(\marginx, \marginy)$-separated, we say that $\cD$ has margin $\marginx$, and label margin $\marginy$. The conditions of well-separated distributions are met in most practical cases, where classification is rarely ambiguous, and arbitrarily close examples are usually classified identically.

\section{The Statistical IRM Assumption}\label{sec:phi_covariates}
Generalizing from source data in a feature space induced by some $\phi \in \Phi$ is only possible if $\Phi$ contains a map that appropriately unifies the classification tasks on $\Ds$ and $\Dt$. We use this section to motivate and define some desirable properties of feature maps vis a vis this goal, and to introduce the \Assumption, formalizing our requirement for the existence of quality maps in $\Phi$.

\subsection{Desirable Properties of Feature Maps}
In Invariant Risk Minimization, the fundamental assumption is the existence of a feature map $\psi^*: \cX \to \cZ$ and an ``invariant predictor'' $h^*: \cZ \to \cY$ for which $h^* \circ \psi^*$ is Bayes-optimal on all training and testing environments. This allows a learner to assume that selecting a feature space through which good performance on training environments is attainable is not a completely futile approach to constructing a generalizing classifier. 

In this spirit, we first interest ourselves in feature maps which preserve the possibility of optimal classification on our single source distribution. We consider a slightly stronger but natural notion that encodes the idea that no information relevant to the classification task on the source should be lost under the mapping. 

\begin{definition}\label{defn:source_preserves}
We say a feature map $\phi$ \textbf{source-preserves} if the induced source distribution $\Ds^\phi$ is separated, and the Bayes risk on $\Ds^\phi$ equal to that of $\Ds$, i.e.
\begin{equation*}
R(\bayes_{\Ds^\phi}, \Ds^\phi) = R_s^*. 
\end{equation*}
Let $\Phis$ denote the set of all source preserving feature maps in $\Phi$.
\end{definition}

Thus, source-preserving feature maps retain all information needed for optimal classification in the sense that the risk of the Bayes optimal in original space $\cX$ and feature space $\cZ$ should be the same under the correct embedding.  We also require that some margin is preserved in the arising feature space.  

While not strictly necessary under the IRM assumption, it also desirable that an embedding maps examples that are similar with respect to the classification task to similar parts of the feature space, regardless of which distribution they come from. 
We formalize a condition capturing this idea via the following.

\begin{definition}\label{defn:covers}
We say a feature map $\phi$ \textbf{contracts}  $\Ds$ and $\Dt$ if the induced source $\Ds^\phi$ is separated with margin $\marginx^\phi$, and for each $z_t \in \supp(\mu_t^\phi)$, there is some $z_s \in \supp(\mu_s^\phi)$ such that 
\begin{equation*}
d_{\cZ}(z_t, z_s) < \frac{\marginx^\phi}{\coverconst}, 
\end{equation*} 
where $\coverconst > 2$ is a fixed constant. Let $\Phic$ denote the set of all contracting feature maps in $\Phi$. 
\end{definition} 

Ultimately, we are interested in feature spaces in which we can generalize to the target by classifying target data as we would source data. This possibility is captured by the notion of the invariant predictor $h^*$ in the IRM assumption. We interest ourselves in feature maps with a similar property - ones for which the optimal classification decision is locally the same across $\Ds$ and $\Dt$. 

\begin{definition}\label{defn:joint_classifies}
We say a feature map $\phi$ \textbf{Bayes-unifies} $\Ds$ and $\Dt$ if for all $x_s \in \supp(\mu_s)$ and $x_t \in \supp(\mu_t)$, 
\begin{equation*}
d_\phi(x_s, x_t) < \frac{\rho^\phi}{2} \implies g_{\Ds}(x_s) = g_{\Dt}(x_t).
\end{equation*} 
Let $\Phiu$ denote the set of all Bayes-unifying feature maps in $\Phi$.  
\end{definition}
Under a feature map which Bayes-unifies, any points which are mapped closer together than half the induced margin are classified the same under the source and target distributions. 

\subsection{Stating the Statistical IRM Assumption} 
It's intuitive that if a feature map both preserves the Bayes risk on the source, and unifies the classification tasks of source and target, then converging to the Bayes risk on the target is possible when source data populate the support of the induced target. 

Thus, we would like a feature map which possess \textit{all} of these properties. Our fundamental assumption is that there exists at least one such feature map in $\Phi$ -- we term this the Statistical IRM Assumption.

\begin{assumption}[Statistical IRM Assumption]\label{assum:SIRM}
We assume there is some $\phi^* \in \Phi$ such that 
\begin{enumerate}
\vspace{-1.5mm}
\item $\phi^*$ source-preserves $\Ds$
\vspace{-1.5mm}
\item $\phi^*$ contracts the source $\Ds$ and target $\Dt$
\vspace{-1.5mm}
\item $\phi^*$ Bayes-unifies source $\Ds$ and target $\Dt$
\vspace{-1mm}
\end{enumerate}
We say that $\phi^*$ with all of these properties \textbf{realizes} the \Assumption, and let $\Phirealize$ denote the set of all maps in $\Phi$ which realize the \Assumption.
\end{assumption}

This assumption is an analogue of the IRM assumption, adapted to our single-source, single-target setting. Like IRM, it allows for the possibility of optimal classification on both source and target via the selection of an appropriate feature space. Contraction, which is not an assumption in IRM, allows for that optimal classification to be realized via a local classification scheme such as $k_n$-NN. 

\subsection{The Statistical IRM Theorem}
One would expect that if a learner were handed $\phi^* \in \Phirealize$,  generalization to the target should be possible with source data alone -- because the classification task on $\Ds$ and $\Dt$ is unified in the feature space arising from $\phi^*$, and every example in the target support is mapped close to the training support, the learner able to construct a good classifier for the target by simply constructing a constructing a good classifier on the induced source.

We formalize this intuition via the following theorem, which states that given knowledge of a realizing feature map $\phi^*$, generalization to the target can be accomplished with source data only via the construction of a $k_n$-NN classifier in feature space.

\begin{theorem}[Statistical IRM Theorem]\label{thm:k_nn_converge}
Suppose $\phi^*$ realizes the Statistical IRM assumption. Then for all $\epsilon, \delta > 0$, there exists $N$ such that for all $n \geq N$, with probability $\geq 1-\delta$ over $S \sim \Ds^n$, 
\begin{equation*}
    R(\n_{S}^{\phi^*}, \Dt) \leq R_t^* + \epsilon.
\end{equation*}
\end{theorem}
Thus, from the perspective of target generalization from source data, it suffices to determine a feature map $\phi \in \Phi$ which realizes the \Assumption. In what follows, we characterize the statistical identifiability of $\phi^*$ (and thus the learnability of $\Dt$) under this assumption in each of our data availability settings.

\section{The Distance Dimension of $\Phi$}
Our investigation of the identifiability of realizing feature maps relies on one further notion -- one of embedding classes with bounded complexity. To this end, we introduce a complexity measure on $\Phi$ which will play a key role in each of the settings we consider. We begin with an intermediate definition.
\begin{definition}\label{defn:dist_comparer}
For a given $\phi: \cX \to \cZ$, we define its \textbf{induced distance comparer} $\dist_\phi: \cX^4 \to \{0, 1\}$ as the map 
\begin{equation*}
\dist_\phi(x_1,  x_2, x_3, x_4) = \ind\left( d_{\cZ}\left(\phi(x_1), \phi(x_2)\right) \geq d_{\cZ}\left(\phi(x_3), \phi(x_4)\right)\right).
\end{equation*}
We also define $\dist\Phi := \{\dist_\phi: \phi \in \Phi\}$ as the \textbf{induced distance comparer class} of $\Phi$.
\end{definition}
Distance comparers are a natural tool for our analysis -- all nearest-neighbor computations inside the feature space $\cZ$ can be expressed in such terms. This observation gives rise to a natural complexity measure for the determination of a suitable feature map, which we term the \textit{distance dimension}.

\begin{definition}\label{defn:dist_dim}
The \textbf{distance dimension} of $\Phi$, denoted $\dd(\Phi)$, is the VC dimension of the induced comparer class $\dist\Phi$. 
\end{definition}

In providing upper bounds, it will be important that the distance dimension be finite. We note that it is easily bounded for the two important classes of feature maps mentioned in Section \ref{sec:prelims}.

\begin{theorem}\label{thm:bound_linear_dd}
Suppose $\cor_{D, K}$ and $\proj_{D, K}$ are defined as in Examples \ref{def:cor} and \ref{def:proj}, respectively.  Then 
\begin{equation*}
\dd(\cor_{D, K}) \leq K\log D \ \text{ and  } \ \dd(\proj_{D, K}) \leq D^2.
\end{equation*}
\end{theorem}

\section{Direct Generalization from $\Ds$}

We first study the possibility of constructing a classifier that generalizes to $\Dt$ using only labeled samples from $\Ds$. In this setting, a learner $L$ takes input $S \sim D_s^n$, and outputs a classifier $L(S): \cX \to \cY$, with the goal of achieving a small risk on $\Dt$.

While one might hope that the Statistical IRM assumption alone is sufficient for generalization to the target, this is unfortunately false. In fact, we have already seen an example of this phenomenon in Figure \ref{fig:page_2_fig}(c). Here, we essentially argued that $\phi_y$ was a realizing feature map: it preserves the source risk, unifies the classification tasks on source and target, and maps all target points close to source points. However, we argued that $\phi_x$ and $\phi_y$ were statistically indistinguishable in this setting, leaving the learner in need of more information.


To the end of a general characterization of learnability in this setting, recall our discussion of Figure \ref{fig:page_2_fig}(a). Here, the projection $\phi_y$ could be determined as realizing the Statistical IRM assumption given that it was the only map in $\Phi$ that preserved the source distribution -- it's clear from the figure that that $\phi_x$ does not preserve the source, and so cannot possibly realize the Statistical IRM assumption. It is vital that this reasoning could be carried out with source data alone.

More generally, by the Statistical IRM Theorem, it is sufficient for generalization from source data alone that the learner be able to identify $\phi^*$ \textit{from source data alone}. Such a realizing feature map must of course satisfy all three requirements of Assumption \ref{assum:SIRM}.  However, note that only one of these requirements, namely source-preservation, depends on the source distribution alone --  the others, namely contraction and Bayes-unification, are defined in terms of the target distribution. As such, only source-preservation can be tested using source data. 

That said, if the learner can be assured that \textit{all} source-preserving feature maps realize the Statistical IRM assumption, i.e. $\Phirealize = \Phis$, it can identify realizing feature maps by identify source-preserving feature maps. We formalize this intuitive idea with the following theorem, which shows that PAC guarantees for target generalization are obtainable when the additional condition $\Phirealize = \Phis$ holds.

\begin{theorem}\label{thm:setting_2_upper_bound}
Suppose the Statistical IRM Assumption holds, the distance dimension $\partial(\Phi) <\infty$, and that $\Phirealize = \Phis.$ Then there is a learning rule $L$ such that for every $\epsilon, \delta > 0$, there exists $N$ such that if $n \geq N$,  with probability $\geq 1-\delta$ over $S \sim \Ds^n$, 
\begin{equation*}
R(L(S), \Dt) \leq R_t^* + \epsilon.
\end{equation*}
\end{theorem}
We relegate specification of the learning rule to the Appendix. It is founded on minimizing the empirical risk on the source data over feature maps in $\Phi$, but further leverages the knowledge that source-preserving feature maps induce separated distributions over feature space. After selecting a candidate feature map which empirically matches the requirements for source preservation, it uses $k_n$-NN in the implied feature space to make predictions.

In light of the discussion above, the condition that $\Phirealize = \Phis$ is intuitively necessary as well -- without it, there may be some feature map which is source-preserving, but which e.g. fails to Bayes-unify the source and the target. It's simple to see that attempting to generalize to the target via such a feature space could be catastrophic, and thus that blindly choosing between source-preserving feature maps will eventually lead the learner astray. On the other hand, not classifying through a feature space subjects the learner to the standard pitfalls of out of distribution generalization. 
We formalize these ideas via the following hardness result.

\begin{theorem}\label{theorem:type-1-lower-bound}
Fix a source $\Ds$, a target $\Dt$, and some embedding class $\Phi$ for which the Statistical IRM assumption holds. Suppose that $\Phis \setminus \Phirealize$ is non-empty, and that a learner $L$ successfully generalizes to $\Dt$ (with high probability) using only samples from $\Ds$. Then for all $\epsilon > 0$, there exists data distributions $\Ds', \Dt'$ such that the following hold::
\begin{enumerate}
	\item $W(\Ds, \Ds') < \epsilon$.
	\item There is a $\phi \in \Phi$ realizing the Statistical IRM assumption on alternative source $\Ds'$ and alternative target $\Dt'$.
	\item For all $N$, there exists $n > N$ such that with probability at least $\frac{1}{4}$ over $S \sim \Ds^n$, $$R(L(S), \Dt') > R(g_{\Dt'}, \Dt') + \frac{1}{4}.$$
\end{enumerate}
\end{theorem}

Thus, in the case that some feature maps preserve the source but do not realize the Statistical IRM assumption, there is always some problem nearly identical problem instance where the Statistical IRM assumption is realized by $\Phi$, buy which causes a given learner to have unbounded sample complexity (for some choice of $\epsilon$ and $\delta$).

\section{Combining Labeled Samples from $\Ds$ and Unlabeled Samples from $\Dt$}\label{sec:unlabel}

We now consider the less restrictive unlabeled target data are also available. Here, a learner $L$ takes input $S \sim \Ds^n$ and $U \sim \mu_t^m$, and outputs a classifier $L(S, U):  \cX \to \cY$. 

The story given additional access to unlabeled data is similar to source-only setting: the Statistical IRM assumption alone is insufficient for guaranteeing successful generalization when additional unlabeled target data are available. In other words, the combination of labeled source and unlabeled target is generally insufficient for identifying a $\phi \in \Phi$ that realizes the Statistical IRM assumption. 

For a simple example to this end, we return to panel (c) of Figure \ref{fig:page_2_fig}. For the source and target distributions shown, it is evident that no amount of labeled data from the source and unlabeled data from target will allow us decide whether we should project data onto the $x$-axis or the $y$-axis. This is because the only difference between them is the manner in which $\Dt$ is \textit{labeled}. By contrast, the example depicted by Figure \ref{fig:page_2_fig} panel (b) illustrates a case in which the additional unlabeled data from $\Dt$ proves sufficient: because projecting onto the $x$-axis fails to map target points close source points, we can conclude that $\phi$ must be the projection onto the $y$-axis.

As in source-only setting, identifying a feature map realizing the Statistical IRM assumption requires testing the three conditions of Assumption \ref{assum:SIRM}. Understanding the utility of additional unlabeled target data is to realize that it allows the learner to not only test which feature maps preserve the source, but also which maps contract $\Ds$ and $\Dt$. On the other hand, it is insufficient to determine which feature maps Bayes-unify, as this notion intrinsically depends on labeling under $\Dt$. This motivates a similar sufficient condition for learnability as we saw in the previous section --namely, that all feature maps which both preserve the source and contract the source and target further Bayes unify. The following theorem shows that this is indeed a sufficient condition for learnability in this setting. 

\begin{theorem}\label{thm:unlabel_upper_bound}
Suppose the Statistical IRM Assumption holds, the distance dimension $\dd(\Phi) < \infty$, and that $\Phirealize = \Phis \cap \Phic.$ Then there is a learning rule $L$ such that for all $\epsilon, \delta > 0$, there exist $N$ and $M$, such that if $n \geq N$ and $m \geq M$, with probability $\geq 1-\delta$ over $S \sim \Ds^n$ and $U \sim \mu_t^m$, 
\begin{equation*}
R(L(S, U), \Dt) \leq R_t^* + \epsilon.
\end{equation*}
\end{theorem}
 
 The learning rule is specified in detail in the Appendix. It proceeds by first selecting a feature map which both empirically preserves the source, and maps each unlabeled target in $U$ point close to some source point in $S$.  As above, it uses $k_n$-NN in the selected feature space to make predictions.

In accordance with the intuition developed above, the condition that $\Phirealize = \Phis \cap \Phic$ is necessary. The issue of course is that without access to labeled target data, testing whether a feature map Bayes-unifies is impossible. We formalize this via the following hardness result. 

\begin{theorem}\label{thm:lower_bound_setting_unlabeled}
Suppose $\Phi$ realizes the Statistical IRM Assumption for $\Ds$ and $\Dt$, and that there is some $\phi \in \Phis \cap \Phic$ for which $ \phi \not \in \Phirealize$. Then for all learners $L$, there exists a conditional data distribution, $\eta'$ such that the following hold:
\begin{enumerate}
    \item $\eta'(y|x) = \eta(y|x)$ for all $x \in \supp(\mu_s)$.
    \item $\Phi$ realizes the Statistical IRM Assumption for $\Ds$ and $\Dt' = (\mu_t, \eta')$.
    \item There exists $\delta, \epsilon > 0$ such that for arbitrarily large values of $n$ and $m$, with probability at least $\delta$ over $S \sim \Ds^n$ and $U \sim \mu_t^m$, 
    \begin{equation*}
    R(L(S, U), \Dt') > R(\bayes_{\Dt'}, \Dt') + \epsilon.
    \end{equation*}
\end{enumerate}
\end{theorem}

Theorem \ref{thm:lower_bound_setting_unlabeled} shows that no combination of embedding class $\Phi$ and learner $L$ can circumvent the impossibility of testing Bayes unification with unlabeled target data. For any embedding class and learning algorithm, one can always find a pair of source and target distributions on which $\Phi$ realizes the \Assumption, but on which the learning algorithm will fail.

\section{Efficient Use of Labeled Samples from $\Dt$}

\begin{algorithm}[t!]\label{alg:setting_1}
\caption{Selection of an Appropriate Feature Map via Target Loss Validation}
\begin{algorithmic}[1]
\Procedure{\nERM}{$S \sim \Ds^n$, $T \sim \Dt^m$}
 \State $\hat{\phi} = \argmin_{\phi \in \Phi} \frac{1}{m}\sum_{(x,y) \in T} \ind\left(\n^{\phi}_{S}\neq y\right)$ 
 \vspace{-1.5mm}
 \State \textbf{return } $\n^{\hat{\phi}}_{S}$
\EndProcedure
\end{algorithmic}
\end{algorithm}

The discussion above implies that even under the Statistical IRM assumption, there are many situations where label target data is required for generalization. In such cases, we would hope that we can exploit the information encoded in 
the Statistical IRM Assumption to achieve generalization through labeled source data and a small amount of labeled target data. In this section we show that the Statistical IRM assumption allows for significant convergence rate speed-ups in many settings.

Recall that the Statistical IRM theorem states that given a realizing feature map, generalization to the target can be accomplished purely through source data -- the limitation of a lack is labeled target data is the difficulty in identifying such a feature map.
This inspires the strategy of allocating all of the labeled target data towards determining a realizing feature map.

 In this spirit, we analyze the natural scheme of constructing a classifier by composing nearest-neighbors trained solely on source data with the map $\phi \in \Phi$ that minimizes the empirical risk over $T \sim \cD_t^m$, finding that the number of target examples required for guarantees can be controlled in terms of the ``distance dimension'' of the class $\partial(\Phi)$. 

\begin{theorem}\label{thm:setting_1_upper_bound}
Suppose $\Phi$ realizes the Statistical IRM assumption. Then for every $\epsilon, \delta > 0$, there exists $N$ such that if 
\begin{equation*}
n \geq N, m \geq \Omega\left(\frac{\partial(\Phi)\log \left(n+\partial(\Phi)\right) + \log\frac{1}{\delta}}{\epsilon^2} \right),
\end{equation*}
then with probability at least $1-\delta$ over $S \sim \Ds^n$, $T \sim \Dt^{m}$, 
\begin{equation*}
R(\n^{\hat{\phi}}_{S}, \Dt) \leq R_t^* + \epsilon,
\end{equation*}
where $\n^{\hat{\phi}}_{S}$ is output of $\nERM(S, T)$.
\end{theorem}

Thus, the amount of labeled target data required for generalization when $\Phi$ realizes the Statistical IRM assumption can be largely controlled through our complexity measure on the class $\Phi$. We say ``largely'' given that $m$, the amount of data required from $\Dt$, has a logarithmic dependence on $n$, the amount of data drawn from $\Ds$. This implies a near distributional-independence between source and target in the sample complexity. 

We note that the above margin assumptions are not required for the analysis leading to Theorem \ref{thm:setting_1_upper_bound}.  Thus -- comparing e.g. to rates of convergence under the canonical Tsybakov noise assumption in $\mathcal{X} = \mathbb{R}^d$, under which nonparametric classifiers necessarily incur rates of $\tilde{\Omega}(m^{-1/1+d})$  -- the guarantees of Theorem \ref{thm:setting_1_upper_bound} represent significant convergence rate speed-ups over naively training a non-parametric classifier with target data in many cases where the distance dimension $\partial(\Phi)$ is polynomial in the dimension of the instance space \cite{audibert2011}.

\section{Discussion}

In this work, we study the problem of distribution shift under a variant of the IRM assumption, wherein it is known that a feature map in a class $\Phi$ unifies classification on source and target. We investigate the identifiability of such maps, characterizing learnability in settings where worst-case approaches indicate that learning should be impossible or expensive. 

Our work suggests that the study of IRM-like assumptions is a promising direction for shedding light on new situations where guaranteeing generalization under distribution shift is possible. It also highlights that a primary issue in learning under IRM-like assumptions may be the statistical identifiability of suitable feature maps. 

\textbf{Acknowledgements: }
This work was supported by the National Science Foundation under the following grants: NSF CIF-2402817, SaTC-2241100, CCF-2217058, and ARO-MURI W911NF2110317. 

RB was also partially supported by the German Research
Foundation through the Cluster of Excellence “Machine
Learning - New Perspectives for Science" (EXC 2064/1 number
390727645)

\bibliographystyle{unsrt}
\bibliography{main}


\newpage
\appendix
\section{Further Notation}\label{app:notation}
Given $x \in \cX$, we let $B(x, r) = \{x': d(x, x') \leq r\}$ denote the closed ball centered at $x$ of radius $r$. For a feature map, $\phi \in \Phi$, we also let $B_\phi(x, r) = \{x': d_\phi(x ,x') \leq r\}$ denote the set of all points with distance (under $\phi$) at most $r$ from $x$.

Recall that for $\phi \in \Phi$, we let $d_\phi$ denote the metric over $\cX$ induced by $\phi$, i.e. $d_\phi(x, x') = d\left(\phi(x), \phi(x')\right)$. We extend this in the natural way to sets, letting 
$$
d_\phi(A, B) = \inf_{a \in A, b \in B}d_\phi(a, b).$$  
Finally, for a pair of feature maps $\phi, \phi'$, we let $d(\phi, \phi')$ denote the supremum metric between $\phi$ and $\phi'$. That is, $$d(\phi, \phi') = \sup_{x \in \cX}d_{\cX}\left(\phi(x), \phi'(x)\right).$$

\section{Further Technical Assumptions}\label{app:measure}

\subsection{Lebesgue Differentiation Theorem}
 We assume that for any finite Borel measure $\mu$ over $\cX$, the Lebesgue differentiation theorem holds. That is, for all measurable functions $f: \cX \to \R$, up to a null set under $\mu$, 
\begin{equation*}
\lim_{r \to 0^+} \frac{1}{\mu(B(x,r))}\int f(x)d\mu(x) = f(x).
\end{equation*}

\subsection{Open Measures}
As referenced in Section \ref{sec:prelims} above, we assume that our Borel measures satisfy a further regularity condition -- namely, that they are open measures. 
\begin{definition}\label{def:open_measure}
A Borel measure, $\mu$, over metric space $(\mathcal{M}, d)$ is \textbf{open}, if for all measurable sets $A$, $\mu(A) > 0$ if only if there exists $x \in \mathcal{M}$ and $r > 0$ such that $B(x, r) \subseteq A$ and $\mu(B(x(x, r)) > 0$. 
\end{definition}

A very typical example of such a measure is any distribution that has a finite density function. In this work, we will restrict ourselves to considering open measures with the following assumption: $\mu_s$ and $\mu_t$ are open, and for all $\phi \in \Phi$, the induced source and target measures, $\mu_s^\phi$ and $\mu_t^\phi$ are open over the metric space $(\phi(\cX), d) \subseteq (\cZ, d)$. Here we are noting that $\mu_s^\phi$ and $\mu_t^\phi$ are only non-zero over subsets of the image of $\phi$, $\phi(\cX)$, and thus we restrict our attention to $\phi(\cX)$ when considering openness.

This technical assumption allows us to simplify our results as it prohibits cases in which $\mu_t^\phi$ can be a pathological distribution that concentrates in an area of $\supp(\mu_s^\phi)$ that leads to bad generalization. We also believe that such an assumption is relatively mild -- all distributions over $\cZ$ are arbitrarily close to open Borel measures -- we can simply add spherical noise to each sampled point. 

\subsection{Assumptions on $\Phi$}\label{app:phi_assumptions_new}

We include two further technical assumptions about $\Phi$. We begin by assuming that all feature maps send an \textit{infinite} number of points to a given $z \in \mathcal{Z}$. 

\begin{assumption}\label{assumption:infinite_inverse}
For all $x \in \mathcal{X}$ and all $\phi \in \Phi$, the set of points $x'$ that have the same image as $x$ in $\mathcal{Z}$ under $\phi$ is infinite. That is, $$|\{x': \phi(x') = \phi(x)\}| = \infty.$$
\end{assumption}

Observe that this assumption is clearly met by the examples given in Section \ref{sec:feature_maps_in_main_body}. Furthermore, it is likely to be met by any reasonable family of continuous maps that perform any kind of dimension reduction.

Next, we define \textit{dominance}, which will be useful for formulating our other assumption.

\begin{definition}
We say that feature map $\phi_1$ \textbf{dominates} feature map $\phi_2$ at point $x$ if $$\{x': \phi_1(x') = \phi_1(x)\} \supseteq \{x': \phi_2(x') = \phi_2(x)\}.$$ 
\end{definition}

We now define an embedding class to be \textit{indomitable} when it avoids instances of one feature map dominating another.

\begin{definition}
$\Phi$ is \textbf{indomitable} if for all distinct $\phi_1, \phi_2 \in \Phi$ and for all $x \in \mathcal{X}$, the following holds. For all $\epsilon > 0$, there exists maps $\phi_1^\epsilon, \phi_2^\epsilon \in \Phi$ such that:
\begin{enumerate}
	\item $d(\phi_1, \phi_1^\epsilon), d(\phi_2, \phi_2^\epsilon) < \epsilon$.
	\item $\phi_1^\epsilon$ does \textit{not} dominate $\phi_2^\epsilon$ at $x$. 
	\item $\phi_2^\epsilon$ does \textit{not} dominate $\phi_1^\epsilon$ at $x$. 
\end{enumerate}
\end{definition}

We will now assume that $\Phi$ is indeed indomitable. 
\begin{assumption}\label{assumption:indomitable}
$\Phi$ is indomitable. 
\end{assumption}

Observe that this assumption is satisfied by both examples of feature maps given in Section \ref{sec:feature_maps_in_main_body}. More generally, the fact that our definition permits a lack of dominance to hold for \textit{some} two maps that are close to $\phi_1$ and $\phi_2$ makes our definition mild enough to hold for most continuous classes of feature maps.

%
%

\section{$k_n$-nearest neighbors}\label{app:knn}

First, we fix $k_n$ as a sequence of integers with the following properties.
\begin{definition}
Let $k_n$ be a sequence of integers so that $\lim_{n \to \infty} \frac{k_n}{\log n} = \infty$, and $\lim_{n \to \infty} \frac{k_n}{n} = 0$. 
\end{definition}
Observe that $k_n = \log^2n$ would suffice as an example of such a series.

Next, our goal is to define the $k_n$-nearest neighbors classifier over a labeled data set of of $n$ points, $S = \{(x_1, y_1), \dots, (x_n, y_n)\}$. To do so, we begin by describing a tie-breaking procedure used in cases where training points are equidistant from a given test point. 

\begin{definition}
An \textbf{ordering} $\pi$, over a dataset $S = \{(x_1, y_1), \dots, (x_n, y_n)\}$ is any ordered permutation of $S$. We say that $(x_i, y_i) <_\pi (x_j, y_j)$ if $(x_i, y_i)$ occurs before $(x_j, y_j)$ in the permutation.
\end{definition}

We now show how to use $\pi$ to break ties when computing nearest neighbors. 

\begin{definition}\label{defn:pi_induced_ordering}
Let $x \in \cX$. Let $\pi$ be an ordering over dataset $S$. For $(x_i, y_i), (x_j, y_j) \in S$, we say that $d(x, x_i) <_\pi d(x, x_j)$ if either of the two conditions hold:
\begin{enumerate}
	\item $d(x, x_i) < d(x, x_j)$.
	\item $d(x, x_i) = d(x, x_j)$ and $(x_i, y_i) <_\pi (x_j, y_j)$. 
\end{enumerate}
\end{definition}

In essence, ties are broken by choosing the datapoint that appears earlier in the ordering. We now define a nearest neighbor as follows.

\begin{definition}\label{defn:pi_k_n_nearest_neighbors}
Let $S = \{(x_1, y_1), \dots, (x_n, y_n)\}$ be a dataset, and let $\pi$ be an ordering of $S$. For $x \in \cX$, we say that $(x_i, y_i)$ is a $k_n$-nearest neighbor of $x$ if $$|\{j: d(x, x_j) <_\pi d(x, x_i)\}| < k_n.$$ We also let $S_{k_n}^\pi(x)$ denote the set of all $k_n$-nearest neighbors of $x$ when using the ordering $\pi$.
\end{definition}

Observe that by construction, $|S_{k_n}^\pi(x)| = k_n$. This is because the ordering $<_\pi$ allows us to strictly order points based on their distances from $x$ with ties broken by $\pi$.

We are now ready to define the nearest neighbors classifier. 
\begin{definition}
Let $S = \{(x_1, y_1), \dots, (x_n, y_n)\}$ be a dataset, and $\pi$ an ordering over $S$. Then for $x \in \cX$, we define $$\n_{S, \pi}(x) = \argmax_{y \in \cY} \sum_{(x_i, y_i) \in S_{k_n}^\pi(x)} \ind(y = y_i).$$ Here, we break ties in $\cY$ arbitrarily (which could be done with an ordering of $\cY$). 
\end{definition}

Throughout the paper, we typically ommit $\pi$ from our notation for $\n_{S, \pi}(x)$. This is because in all cases, we assume that some ordering $\pi$ is implicitly chosen (independent of the data points) ahead of time. 

\subsection{Composing with feature maps.}

We now define the classifier $\n_S^\phi$, where $\phi: \cX \to \cZ$ is a feature map. One important detail for doing so, is that we will continue to use an ordering over $S$, rather than an ordering over $\phi(S) = \{(\phi(x_1), y_1), \dots (\phi(x_n), y_n)\}$. This will allow us to use a single ordering throughout all of our learning algorithms that deal with learning a feature map. 

Recall that for any feature map, $\phi$, $d_\phi: \cX^2 \to [0, \infty)$ denotes the distance metric $$d_\phi(x, x') = d_\cZ\left(\phi(x), \phi(x')\right).$$
Using this, we give analogs to Definitions \ref{defn:pi_induced_ordering} and \ref{defn:pi_k_n_nearest_neighbors} by essentially replacing $d$ with $d_\phi$.

\begin{definition}\label{defn:pi_induced_ordering}
Let $x \in \cX$ and $\phi \in \Phi$. Let $\pi$ be an ordering over dataset $S$. For $(x_i, y_i), (x_j, y_j) \in S$, we say that $d_\phi(x, x_i) <_\pi d_\phi(x, x_j)$ if either of the two conditions hold:
\begin{enumerate}
	\item $d_\phi(x, x_i) < d_\phi(x, x_j)$.
	\item $d_\phi(x, x_i) = d_\phi(x, x_j)$ and $(x_i, y_i) <_\pi (x_j, y_j)$. 
\end{enumerate}
\end{definition}

\begin{definition}\label{defn:pi_k_n_nearest_neighbors}
Let $\phi \in \Phi$. Let $S = \{(x_1, y_1), \dots, (x_n, y_n)\}$ be a dataset, and let $\pi$ be an ordering of $S$. For $x \in \cX$, we say that $(x_i, y_i)$ is a $k_n$-nearest neighbor of $x$ under $\phi$ if $$|\{j: d_\phi(x, x_j) <_\pi d_\phi(x, x_i)\}| < k_n.$$ We also let $S_{k_n, \phi}^\pi(x)$ denote the set of all $k_n$-nearest neighbors of $x$ when using the ordering $\pi$.
\end{definition}

Finally, we define $\n_S^\phi$ as follows.

\begin{definition}
Let $\phi \in \Phi$, let $S = \{(x_1, y_1), \dots, (x_n, y_n)\}$ be a dataset, and $\pi$ an ordering over $S$. Then for $x \in \cX$, we define $$\n_{S, \pi}^\phi(x) = \argmax_{y \in \cY} \sum_{(x_i, y_i) \in S_{k_n, \phi}^\pi(x)} \ind(y = y_i).$$ Here, we break ties in $\cY$ arbitrarily (which could be done with an ordering of $\cY$). 
\end{definition}

The key point of this definition is that all tie-breaking mechanisms are done \textit{independently} of $\phi$. In particular, we have the following.

\begin{lemma}\label{lemma:ties_dont_matter}
Let $S$ be a dataset of $n$ points, and $\pi$ an ordering over $S$. Let $\phi, \phi'$ be two features maps in $\Phi$. Suppose for $x \in \cX$ that for all $i, j$, $d_\phi(x, x_i) \leq d_\phi(x, x_j)$ if and only if $d_{\phi'}(x, x_i) \leq d_{\phi'}(x, x_j)$. Then $\n_{S, \pi}^\phi(x) = \n_{S, \pi}^{\phi'}(x)$.
\end{lemma}

\begin{proof}
This is immediate from the previous definitions as all ties are broken in an identical manner for both $\phi$ and $\phi'$. 
\end{proof}

As before, to avoid cumbersome notation, we will assume that an ordering, $\pi$, of $S$ is fixed. 

%
%
%

\section{Induced conditional distributions}\label{app:Lebesgue}

In this section, we rigorously define the conditional data distribution of $\cD^\phi$. Recall that if $(X, Y) \sim \cD$ denote the random variables corresponding to $\cD$, then $\cD^\phi$ is defined as the data distribution $(\phi(X), Y)$, where $\phi: \cX \to \cZ$ is a feature map. We write $\cD = (\mu, \eta)$, where $\mu$ denotes the measure corresponding to $X$ over $\cX$, and $\eta$ is the conditional data distribution, $p(y|X)$. Our goal in this section is to similarly write $\cD^\phi = (\mu^\phi, \eta^\phi)$.

First, observe that for any measurable subset $B \subseteq \cZ$, $\mu^\phi(B) = \mu\left(\phi^{-1}(B)\right)$. This directly follows from the definition of the random variable $\phi(X)$. 

Next, to define $\eta^\phi$, first recall that $\eta(y|x)$ denotes the probability that $Y=y$ given that $X=x$. By assumption this is well defined for all $x \in \cX$ and $y \in \cY$, and moreover for any $y \in \cY$ the function $\cX \to [0, 1]$ defined by $x \mapsto \eta(y|x)$ is measurable. To define $\eta^\phi$, we first define $\upsilon^y$ for all $y \in \cY$ as follows. 
\begin{definition}
$\upsilon^y$ is a measure over $\cZ$ so that for all measurable sets $B$, $$\upsilon^y(B) = \int_{\phi^{-1}(B)}\eta(y|x)d\mu(x).$$ 
\end{definition}
The fact that $\upsilon^y$ is a well-defined measure follows directly from the rules of integration. In essence, $\upsilon^y(B)$ is the probability. of observing $(X, Y)$ with $\phi(X) \in B$ and $Y = y$. We now show the following: 
\begin{lemma}
$\upsilon^y$ is absolutely continuous with respect to $\mu^\phi$ for all $y$
\end{lemma}
\begin{proof}
This immediately follows from the fact that $\eta(y|x) \leq 1$ for all $y, x$. Thus for any measurable set $B$, $$\upsilon^y(B) = \int_{\phi^{-1}(B)}\eta(y|x)d\mu(x) \leq \int_{\phi^{-1}(B)}d\mu(x) = \mu(\phi^{-1}(B)) = \mu^\phi(B).$$ Thus for any $\epsilon > 0$, we can simply choose $\delta = \epsilon$ so that $\mu^\phi(B) < \delta \implies \upsilon^y(B) < \epsilon$. 
\end{proof}
We now use the Radon-Nikoym theorem on $\upsilon^y$ to define $\eta^\phi$.
\begin{lemma}
For all $y \in \cY$, there exists a measurable function $f^y: \cZ \to [0, 1]$ such that $$\upsilon^y(B) = \int_{B}f^y(z)d\mu^\phi(z),$$ for all measurable sets $B$. 
\end{lemma}
\begin{proof}
This directly follows from the Radon-Nikoym theorem.
\end{proof}
We then define $\eta^\phi$ using these functions, $f^y$.
\begin{definition}
For all $z \in \cZ$ and $y \in \cY$, we define $\eta^\phi(y|z) = f^y(z)$. 
\end{definition}

\section{Technical Lemmas}

%
%
%
%

\subsection{Useful bounds related to $\Phi$}

We now prove several results regarding the distance dimension of $\Phi$, $\partial(\Phi)$. These will be useful for proving all of our subsequent results. 

We begin by defining a useful hypothesis class for analyzing nearest neighbors.

\begin{definition}\label{defn:h_s_phi}
Let $S = \{(x_1^*, y_1^*), \dots, (x_n^*, y_n^*)\}$ be a set of $n$ labeled points in $\cX \times \cY$. For $\phi \in \Phi$, define $h_{S, \phi}: \cX \times \cY \to \{0, 1\}$ as $$h_{S, \phi}\left((x, y)\right) = \begin{cases} 1 & \n_S^\phi(x) = y \\ 0 & \text{otherwise}\end{cases}.$$ Finally, we let $\cH(S, \Phi) = \{h_{S, \phi}: \phi \in \Phi\}.$ 
\end{definition}

Observe that $\cH(S, \Phi)$ is a set of binary classifiers. We now show that it has bounded VC-dimension.

\begin{lemma}\label{lem:vc_bound_classification}
There exists an absolute constant, $c_1 > 0$ such that $\cH(S, \Phi)$ has VC-dimension bounded as $$vc(\cH(S, \Phi)) \leq c_1\partial(\Phi)\log\left(n + \partial(\Phi)\right).$$
\end{lemma}

\begin{proof}
Suppose,$\cH(S, \Phi)$ shatters a set $V$ of $v$ points in $\cX \times \cY$, $V = \{(x_1, y_1), \dots, (x_v, y_v)\}.$ The key observation is that for any $h_\phi \in H_{\hat{S}}$, the way $h_\phi$ labels a given point $(x,y)$ is determined by the $k_n$-nearest neighbors of $\phi(x)$ in $\{\phi(x_1), \dots, \phi(x_n)\}$. Furthermore, by Lemma \ref{lemma:ties_dont_matter}, these labels are full determined by the set of all $\binom{n}{2}$ comparisons, $$\left\{\ind \left(d(\phi(x, x_i) \geq d(\phi(x, x_j)\right): 1 \leq i < j \leq n \right\}.$$ These indicator variables precisely correspond to the definition of a distance comparer (Definition \ref{defn:dist_comparer}). It follows that the number of distinct ways that $\cH(S, \Phi)$ can label $V$ is at most the number of ways $\dist\Phi$ can label all $v\binom{n}{2}$ possible comparisons, $\{(x_i, x_j, x_k): 1 \leq i \leq v, 1 \leq j < k \leq n\}.$ Since by definition, $vc(\dist\Phi) = \partial(\Phi)$, By Sauer's Lemma, the number of ways $\cH(S, \Phi)$ can label $V$ is at most $\left(v\binom{n}{2}\right)^{\partial(\Phi)}$. However, since $\cH(S, \Phi)$ shatters $V$, there exist precisely $2^v$ such labelings. It follows that $v \leq \log \left(\left(v\binom{n}{2}\right)^{\partial(\Phi)}\right)$. From here, straightforward algebra implies that $v = O\left(\partial(\Phi) \log \left(n + \partial(\Phi)\right)\right)$, as desired. 
\end{proof}

Next, we define a hypothesis class that will be useful for bounding the margin of a data distribution.

\begin{definition}\label{defn:the_o_g_q_stuff}
For $\phi \in \Phi$ and $r> 0$, define $\cq_{\phi, r}: \cX^2 \to \{0, 1\}$ as the map $$\cq_{\phi, r}(x, x') = \begin{cases} 1 & d_\phi(x, x') < r \\ 0 &\text{otherwise.}\end{cases}.$$ Let $\cQ(\Phi) = \{\cq_{\phi, r}: \phi \in \Phi, r \in [0, \infty)\}$.  
\end{definition}

Roughly speaking, the class $\cQ(\Phi)$ will prove useful in allowing us to uniformly bound measured distances over a data distribution. We now bound its VC-dimension as follows.

\begin{lemma}\label{lem:margin_vc_bound}
There exists an absolute constant, $c_2 > 0$ such that $\cQ(\Phi)$ has VC-dimension bounded as $$vc(\cQ(\Phi)) \leq c_2\partial(\Phi)\log\left(\partial(\Phi)\right).$$
\end{lemma}

\begin{proof}
Suppose $\cQ(\Phi)$ shatters the set $X = \{(x_1, x_1'), (x_2, x_2'), \dots, (x_v, x_v')\}$. We say that $\phi$ induces ordering, $\geq_\phi$ over $X$ by ranking the pairs in increasing distance. That is, $$(x_i, x_i') \geq_\phi (x_j, x_j') \longleftrightarrow d_\phi(x_i, x_i') \geq d_\phi(x_j, x_j').$$ Our strategy is to double count the number of distinct orderings, $\geq_\phi$, over $X$ that can be constructed using $\phi$. Here, two orderings are distinct if they ever differ for some pair of entries from $X$. 

First, suppose that $v$ is even (which we can assume by deleting a pair from $X$ if needed). Since $\cQ(\Phi)$ shatters $X$, for all $S \subset X$ with $|S| = \frac{v}{2}$, there exists $\phi_S , r$ such that $$d_\phi(x_i, x_i') \leq r \leftrightarrow (x_i, x_i') \in S.$$ Observe that for $S \neq S'$, $\phi_S$ and $\phi_{S'}$ must induce distinct orderings over $X$ as the bottom $v/2$-elements of their orderings are distinct. Since there are $\binom{v}{v/2} \geq 2^{v/2}$ choices for $S$, this shows that there are at least $2^{v/2}$ orderings. 

Second, there are $v^2$ possible quadruples, $(x_i, x_i', x_j, x_j')$. Suppose that $\phi$ and $\phi'$ satisfy that $$\Delta_\phi(x_i, x_i', x_j, x_j') = \Delta_{\phi'}(x_i, x_i', x_j, x_j'),$$ for all $i, j$. By the definition of a distance comparer (Definition \ref{defn:dist_comparer}), this implies that $\phi$ and $\phi'$ induces the same ordering over $X$. Thus it suffices to count the number of ways $\Delta\Phi = \{\Delta_\phi: \phi \in \Phi\}$ can label the set of $v^2$ possible quadruples. By Sauer's Lemma, this is at most $bv^{2\partial(\Phi)}$.

Combining our two observations, it follows that $2^{v/2} \leq bv^{2\partial(\Phi)}$. Standard algebra yields that $v \leq c'\partial\Phi\log(\partial(\Phi))$ for some absolute constant $c'$. 
\end{proof}

Finally, for dealing with margins and labels simultaneously, we introduce the following hypothesis class. 

\begin{definition}\label{defn:q_phi_r_S}
Let $S = \{(x_1^*, y_1^*), \dots, (x_n^*, y_n^*)\}$ be a set of $n$ labeled points in $\cX \times \cY$. For $\phi \in \Phi$, define $\cq_{\phi, r, S}: \cX^2 \to \{0, 1\}$ as follows: $$\cq_{\phi, r, S}(x, x') = \begin{cases}1 & d_\phi(x, x') < r\text{ and }\n_S^\phi(x) \neq \n_S^\phi(x') \\ 0 & \text{otherwise} \end{cases}.$$ We let $\cQ(\Phi, S) = \{\cq_{\phi, r, S}: \phi \in \Phi, r \in [0, \infty)\}.$ 
\end{definition}

We now bound its VC-dimension.

\begin{lemma}\label{lem:complicated_class_vc_bound}
There exists an absolute constant, $c_3 > 0$ such that $\cQ(\Phi, S)$ has VC-dimension bounded as $$vc(\cQ(\Phi, S)) \leq c_3\partial(\Phi)\log\left(n + \partial(\Phi)\right).$$
\end{lemma}

\begin{proof}
Suppose $\cQ(\Phi, S)$ shatters $V = \{(x_1, x_1'), \dots (x_v, x_v')\}$. We will double count the number of subsets of $V$ that can be obtained as the pre-image of $1$ under some $\cq_{\phi, r, S} \in \cQ(\Phi, S)$. For any $\phi \in \Phi$, define $t_\phi: \cX^2 \to \{0, 1\}$ as $$t_\phi((x, x')) = \begin{cases} 1 & \n_S^\phi(x) \neq \n_S^\phi(x') \\ 0 & \text{otherwise}\end{cases}.$$ Then the key observation is that $$q_{\phi, r, S}(x, x') = t_{\phi}(x, x')\cq_{\phi, r}(x, x').$$ Thus, a subset of $\{(x_1, x_1'), \dots, (x_v, x_v')\}$ is the pre-image of $1$ under $\cq_{\phi, r, S}$ if it is precisely the intersection of the pre-images of $1$ under $\cq_{\phi, r}$ and $t_\phi$. 

By Sauer's Lemma and Lemma \ref{lem:margin_vc_bound}, there are at most $O\left(v^{c_2\partial(\Phi)\log(\partial(\Phi)}\right)$ subsets that are the pre-image of $1$ under some $\cq_{\phi, r}$.

We now similarly bound the pre-images under $t_\phi$. To this end, observe that the value of $t_\phi$ over all $(x_i, x_i')$ is completely determined by the way in which $\n_S^\phi$ classifies $\{x_1, \dots, x_v, x_1' \dots, x_v'\}$. This quantity, in turn, is fully determined by the way in which $h_{s, \phi}$ (Definition \ref{defn:h_s_phi}) labels the set $\{x_1, \dots, x_v, x_1', \dots, x_v'\} \times \cY$. Thus, applying Sauer's Lemma along with Lemma \ref{lem:vc_bound_classification}, we see that at most $O\left((2v|\cY|)^{c_1\partial(\Phi)\log\left(n + \partial(\Phi)\right)}\right)$ possible subsets. 

However, since $\cQ(\Phi, S)$ shatters $V$, we know that $2^v$ subsets can be formed in this manner. Thus, it follows that for some constant $c_3'$, $$2^v \leq c_3'\left(v^{c_2\partial(\Phi)\log(\partial(\Phi)}\right)\left((2v|\cY|)^{c_1\partial(\Phi)\log\left(n + \partial(\Phi)\right)}\right).$$ Taking logs and applying standard algebraic manipulations yields the desired result. 
\end{proof}

We end with one final useful hypothesis class that is a generalization of Definition \ref{defn:the_o_g_q_stuff}.
\begin{definition}\label{defn:upgraded_q_stuff}
Let $n > 1$ be an integer. For $\phi \in \Phi$ and $r> 0$, define $\cq_{\phi, r, n}: \cX^{n+1} \to \{0, 1\}$ as the map $$\cq_{\phi, r}(x, x_1, \dots, x_n) = \begin{cases} 1 & \exists_{1 \leq i \leq n}d_\phi(x, x_i) < r \\ 0 &\text{otherwise.}\end{cases}.$$ Let $\cQ_n(\Phi) = \{\cq_{\phi, r, n}: \phi \in \Phi, r \in [0, \infty)\}$.  
\end{definition}

This class will assist us with computing the distance between the source and target distributions simultaneously over all embeddings. 

\begin{lemma}\label{lem:margin_vc_bound_FINAL}
There exists an absolute constant, $c_4 > 0$ such that $\cQ_n(\Phi)$ has VC-dimension bounded as $$vc(\cQ_n(\Phi)) \leq c_4\partial(\Phi)\log\left(\partial(\Phi)+n\right).$$
\end{lemma}

\begin{proof}
Suppose $\cQ_n(\Phi)$ shatters $V= ((x_1, x^1), (x_2, x^2), \dots, (x_v, x^v))$ where $x^i = (x_1^i, x_2^i, \dots, x_n^i) \in \cX^n$. The key observation is that the subset shattered by $\cq_{\phi, r, n}$ is precisely determined by the behavior of $\cq_{\phi, r}$ over the set of $nv$ pairs, $(x_i, x_j^i)$ for $1 \leq i \leq v$ and $1 \leq j \leq n$. By Sauer's Lemma along with Lemma \ref{lem:margin_vc_bound}, there are at most $O((nv)^{c_2\partial(\phi)\log \partial(\Phi)})$ such subsets possible.

It follows that $v \leq C (nv)^{c_2\partial(\phi)\log \partial(\Phi)}$. Applying standard manipulations again yields that $v \leq c_4\partial(\Phi)\log\left(\partial(\Phi)+n\right)$ for some constant $c_4$.
\end{proof}

\subsection{Useful properties of data distributions}

\begin{lemma}\label{lem:stuff_close_to_bayes_good}
Let $\cD$ be a well-separated data distribution with label margin, $\Delta$. Let $h: \cX \to \cY$ be a classifier such that $R(h, \cD) < R(g_\cD, \cD) + \epsilon$, where $\cD$ denotes the Bayes-optimal classifier. Then $$\Pr_{(x,y) \sim \cD}[h(x) \neq g_\cD(x)] < \frac{\epsilon}{\Delta}.$$
\end{lemma}

\begin{proof}
Let $\cD = (\mu, \eta)$. Let $A \subset \cX$ denote the set of all points for which $h$ and $g_\cD$ disagree. Then we have,
\begin{equation*}
\begin{split}
R(h, \cD) &= 1 - \int_{\cX}\eta(h(x)|x)d\mu(x) \\
&= 1 - \int_{A}\eta(h(x)|x)d\mu(x) - \int_{\cX \setminus A}\eta(h(x)|x)d\mu(x) \\
&\geq 1 - \left(\int_A \left(\eta(\bayes_\cD(x)|x) - \Delta\right)d\mu(x) \right) - \left(\int_{\cX \setminus A} \eta(\bayes_\cD(x)|x)d\mu(x)\right) \\
&= 1 + \Delta\mu(A) - \int_{\cX}\eta(\bayes_\cD(x)|x)d\mu(x) \\
&= R(\bayes_{cD}, \cD) + \Delta\mu(A)
\end{split}
\end{equation*}
Here, we are using the fact that $\eta(h(x)|x) \leq \eta(g_{\cD}(x)|x) - \Delta$ for $g_{\cD}(x) \neq h(x)$ (as $\cD$ is well-separated). Finally, using the fact that $h$ has excess risk at most $\epsilon$, we find that $\Delta \mu(A) < \epsilon$ which implies that $\mu(A) < \frac{\epsilon}{\Delta}$, as desired. 
\end{proof}

We now use the fact that $\Phi$ is compact to prove a useful Lemma.

\begin{lemma}\label{lem:compact_ABUSE}
For all $\epsilon > 0$, there exists $\delta > 0$ such that for \textit{all} $\phi \in \Phi$ and \textit{all} $x \in \cX$, $$\phi\left(B(x, \delta)\right) \subseteq B(\phi(x), \epsilon).$$
\end{lemma}

\begin{proof}
Assume towards a contradiction, that for some $\epsilon > 0$, for all $\delta > 0$, there exists $\phi, x$ such that $\phi\left(B(x, \delta)\right) \not \subseteq B(\phi(x), \epsilon)$. Let $\delta_i \to 0$ be a sequence and let $\phi_i, x_i$ be corresponding feature maps and points for this sequence. 

Since $\Phi$ is compact, we can take an infinite subsequence of $\phi_i$ so that $\phi_i \to \phi$ for some $\phi$. Similarly, since $\cX$ is compact, we can take an infinite subsequence so that $x_i \to x$ for some $x$. Because $\phi$ is continuous, there exists $\delta > 0$ such that $$\phi\left(B(x, \delta) \right) \subseteq B(\phi(x), \frac{\epsilon}{2}).$$ Select $i$ such that $d(x, x_i) < \frac{\delta}{2}$, $\delta_i < \frac{\delta}{2}$, and $d(\phi, \phi_i) < \frac{\epsilon}{2}$. Then, applying the triangle inequality, we have
\begin{equation*}
\begin{split}
B(x_i, \delta_i) &\subseteq B(x_i, \frac{\delta}{2}) \\
&\subseteq B(x, \delta).
\end{split}
\end{equation*}
Furthermore, since $d(\phi, \phi_i) < \frac{\epsilon}{2}$,
\begin{equation*}
\begin{split}
\phi_i(B(x, \delta)) &\subseteq \{z: d_\cZ\left(z, \phi(B(x, \delta))\right) < \frac{\epsilon}{2}\} \\
&\subseteq \{z: d_\cZ\left(z, B(\phi(x), \frac{\epsilon}{2})\right) < \frac{\epsilon}{2}\} \\
&\subseteq B(\phi(x), \epsilon)
\end{split}
\end{equation*}
However, this is a contradiction to the definition of $\delta_i$.
\end{proof}

\subsection{Useful Definitions for Analyzing Margins}

We begin by precisely characterizing feature maps that preserve a data distribution $\cD$.

\begin{lemma}\label{lem:characterize_well_sep}
Let $\cD$ be a well-separted distribution, and let $\{\mu^y: y \in \cY\}$ be the sets as defined in Definition \ref{defn:well-separated}. Then $\phi$ preserves $\cD$ if and only if there exists $\rho^\phi > 0$ such that $$\min_{y \neq y'}d_\phi(\mu^y, \mu^{y'}) = \rho^\phi.$$
\end{lemma}

\begin{proof}
The first direction is immediate. If $\rho^\phi$ exists, then it is clear that $\cD^\phi$ is well-separated with corresponding sets $\phi(\mu^y)$.

In the second direction, assume towards a contradiction that no such $\rho^\phi$ exists. Because $\phi$ preserves $\cD$, there exist corresponding sets $\{\mu_\phi^y: y \in \cY\}$ that partition the support of $\cD^\phi$. Because no such $\rho^\phi$ exists, we must have some $x \in \mu^y$ such that $\phi(x) \in \mu_\phi^{y'}$ for $y \neq y'$ -- otherwise we could have used the margin of $\Ds^\phi$ as a valid choice for $\rho^\phi$. 

However, it then becomes clear that there exists a ball of non-zero radius centered at $x$ that is mapped into $\mu_\phi^{y'}$. This means it is classified as $y'$ by $g_{\cD^\phi}$ while it is classified as $y$ by $g_{\cD}$. Since $\cD$ is well-separated, there is a unique Bayes-optimal classifier over the support of $\cD$, and this shows that $g_{\cD^\phi}$ does \textit{not} incur Bayes-optimal risk over $\cD$. Thus $\phi$ does not preserve $\cD$, which is a contradiction. 

\end{proof}

We now generalize the idea of the margin of a data distribution as follows.

\begin{definition}\label{defn:distribution_margin}
Let $\cD = (\mu, \eta)$ be a well-separated data distribution, and let $\phi \in \Phi$ be a feature map. Then the \textbf{margin variable} of $\cD^\phi$, is random variable, $\mv^\phi$ defined as follows. Let $(x, x') \sim \mu^2$. Then $$\mv^\phi = \begin{cases} d(x, x') & \bayes_\cD(x) \neq \bayes_\cD(x') \\ \infty & \text{otherwise}\end{cases}.$$
\end{definition}

$\alpha^\phi$ can be thought of as a randomly observed margin. We will be particularly interested in observing small values of $\alpha^\phi$, as this will be reflective of hte margin of $\cD^\phi$. In particular, we have the following. 

\begin{lemma}\label{lem:pseudo_margin_continous}
Let $\cD$ be a well-separated data distribution and let $\phi \in \Phi$ be a feature map. If $\phi$ preserves $\cD$, let $\marginx^\phi$ denote the margin of $\cD^\phi$. Otherwise, let $\marginx^\phi = 0$. Then for every $\gamma > 0$ there exists $\delta > 0$ such that $$\Pr_{\mv^\phi}[\mv^\phi \leq \marginx^\phi + \gamma] \geq \delta.$$ 
\end{lemma}

\begin{proof}
Let $\cD = (\mu ,\eta)$, and let $\{\mu^y: y \in \cY\}$ be the sets corresponding to Definition \ref{defn:well-separated}. Suppose $\phi$ preserves $\cD$. Then the sets $\{\phi(\mu^y): y \in \cY\}$ must be the corresponding sets for $\cD^\phi$, and it follows that $$\min_{y \neq y'}d_\cZ\left(\phi(\mu^y), \phi(\mu^{y'})\right) = \marginx^\phi.$$ On the other hand, if $\phi$ does \textit{not} preserve $\cD$, then we must have $$\min_{y \neq y'}d_\cZ\left(\phi(\mu^y), \phi(\mu^{y'})\right) = 0,$$ as if this distance were positive, then $\phi$ would clearly preserve $\Ds$. Thus in either case, there exists $y, y'$ so that $d\left(\phi(\mu^y), \phi(\mu^{y'})\right) = \marginx^\phi$.

It follows that there exists $x \in \mu^y$ and $x' \in \mu^{y'}$ such that $d_\phi(x, x') \leq \marginx^\phi + \gamma/2$. Let $$\delta = \mu(B_\phi(x, \gamma/4))\mu(B_\phi(x', \gamma/4)).$$ Because $\phi$ is continuous, $\phi(x)$ and $\phi(y)$ lie within the support of $\Ds^\phi$. It follows that $\delta > 0$. $\delta$ is also a lower bound on the probability that we observe $(x_1, x_2) \in B_\phi(x, \gamma/4) \times B_\phi(x', \gamma/4)$, which means it is a lower bound on the probability that $\mv^\phi \leq \marginx^\phi + \gamma.$ This gives the desired result. 
\end{proof}

We now use a similar idea to describe distances between the supports of two measures.

\begin{definition}\label{defn:beta_distribution_margin}
Let $\mu_s, \mu_t$ be measures over $\cX$. Then $\beta^\phi$ is defined as $$\beta^\phi = d_\phi(x_t, \supp(\mu_s)),$$ where $x_t$ is a random variable following distribution $\mu_t$. 
\end{definition}

$\beta^\phi$ can be thought of as representing the distance that a point drawn from $\Dt$ has from $\Ds$ when using distance metric determined by $\phi$. 

It will also be useful to define finite sample version of $\beta_n$, that don't rely on the sets, $\supp(\mu_s)$.

\begin{definition}\label{defn:beta_distribution_margin_n}
Let $\mu_s, \mu_t$ be measures over $\cX$. Let $n > 0$. Then $\beta_n^\phi$ defined as $$\beta_n^\phi = \min_{1 \leq i \leq n} d_\phi(x_t, x_s^i),$$ where $x_t$ is a random variable following distribution $\mu_t$, and $x_s^1, \dots, x_s^n$ are drawn i.i.d from $\mu$. 
\end{definition}

We now show that $\beta_n$ converges to $\beta$.

\begin{lemma}\label{lemma:beta_n_converges}
Let $\phi$ be any feature map. Then $\beta_1^\phi, \beta_2^\phi, \dots$ converges in distribution to $\beta^\phi$.
\end{lemma}

\begin{proof}
For any $r > 0$, the probability that $\beta^\phi < r$ is precisely the probability that some $x_t \in \supp(\mu_t)$ is chosen so that $x_t$ has distance less than $r$ from $\supp(\mu_s)$. For all such $x_t$, and let $x_s \in \supp(\mu_s)$ satisfy $d(x_t, x_s) < r$. Furthermore pick $\epsilon$ so that $2\epsilon < r - d(x_t, x_s)$. It follows that $\beta_n^\phi < r$ will hold if one of the $n$ points selected from $\mu_s$ will be within distance $\epsilon$ from $x_s$. However, this event occurs with high probability for $n$ being sufficiently large. 
\end{proof}

\section{Proof of Theorem \ref{thm:k_nn_converge}}

First, we characterize areas of $\cX$ that are likely to be correctly classified by composing nearest neighbors with $\phi$. 

\begin{definition}\label{defn:p_delta}
Let $\phi$ be a feature map that preserves $\Ds$. Let $0 < p< 1$, and let $r > 0$ be a distance. We let $\cX_{p, r}^\phi$ denote the set of all points $x$ such that there exists $x'$ for which the following hold.
\begin{enumerate}
	\item $d_\phi(x, x') < \frac{\marginx^\phi}{2} - r$.
	\item $\mu_s\left(B_\phi(x', r)\right) \geq p$. 
\end{enumerate}
\end{definition}

Here $p$ represents a small amount of mass that must be close to $x$, and $x'$ and $r$ determine a region in which that mass is concentrated. The idea will be that $x$ can be accurately classified using points sampled from $B(x', r)$. We now formalize this with the following lemma.

\begin{lemma}\label{lem:k_nn_p_delta_accurate}
Fix $p, r > 0$. Then there exists $N$ such that for all $n \geq N$, for all $\phi \in \Phi$ and $x \in \cX_{p, r}^\phi$ with probability at least $1 - \frac{1}{n^4}$ over $S \sim \Ds^\phi$, $$\n_S^{\phi}(x) = g_{\Ds^\phi}\left(\argmin_{z \in \supp(\Ds^\phi)} d_\cZ(z, \phi(x))\right).$$
\end{lemma}

\begin{proof}
Because $\cD_s$ is well-separated, let $\mu_s^y$ denote the regions that correspond to Definition \ref{defn:well-separated}. Let $x'$ be the point as defined in the definition of $\cX_{p, r}^\phi$. By applying Lemma \ref{lem:characterize_well_sep}, observe that there exists $y \in \cY$ such that $\argmin_{z \in \supp(\Ds^\phi)} \in \phi(\mu_s^y)$, and $x' \in \mu_s^y$. This holds because if it didn't, then the triangle inequality would show that $\phi(\mu_s^y)$ and $\phi(\mu_s^{y'})$ have distance less than $\rho^\phi$. 

It now suffices to show that with probability at least $1 - \frac{1}{n^2}$ over $S \sim \Ds^\phi$, $\n_S^\phi(x) = y$. 

To do this, let $S = \{(x_1, y_1), \dots, (x_n, y_n)\}$, and let $X = \{x_1, \dots, x_n\}$. We can view $S$ as being constructed by first drawing $X$, and then drawing the labels of each of its points.

Observe that if $B_\phi(x', r)$ contains at least $k_n$ points, then the $k_n$ nearest neighbors (according to $d_\phi$) of $x$ will all be drawn from $\mu_s^y$. To this end, by Hoeffding's inequality, we see that $$\Pr\left[|X \cap B_\phi(x', r)| > \frac{np}{2}\right] \geq 1 - \exp\left(\frac{n^2p^2}{2n}\right) = 1 - \exp\left(\frac{np^2}{2}\right).$$ Thus, for $n$ sufficiently large (depending only on $p$), this quantity is at least $1  - \frac{1}{2n^4}$, and $\frac{np}{2} > k_n$. This means that with probability at least $1- \frac{1}{2n^4}$, the $B_\phi(x', r)$ contains at least $k_n$ points. 

Next, suppose that this even occurs. We now select the labels for our points. Because of our method of generating $S$, we can assume that these labels are i.i.d and drawn for points in $\mu^y$. Let the label of the $i$th nearest neighbor of $x$ be denoted as $y_i$. For all $y' \neq y$, define $J_i^{y'}$ as the random variable that is $1$ if $y_i = y$, $-1$ if $y_i = y'$, and $-1$ otherwise. The key observation is that $\n_S^\phi(x) = y$ if and only if $\sum_{i=1}^{k_n} J_i^{y'} > 0$ for all $y' \neq y$ as this will imply that $y$ is the pluarlity choice. 

Because $\cD_s$ is well separated, it has label margin $\Delta$. Therefore, $J_i^{y'}$ is a random variable bounded in $[-1, 1]$ with expected value at least $\Delta$. It follows by Hoeffding's inequality, that $$\Pr[\sum_{i=1}^{k_n} J_i^{y'} > 0] \geq 1 - \exp\left(\frac{-2\Delta^2 k_n^2}{4k_n}\right) = 1 - \exp\left(\frac{-\Delta^2k_n}{2}\right).$$ Because $k_n \geq \omega(\log n)$, it follows that for a sufficiently large value of $n$, this quantity is at least $1 - \frac{1}{2n^4|\cY|}$. Thus taking a union bound over all $y' \in \cY \setminus \{y\}$ gives the desired result. 
\end{proof}

Here observe that we are comparing the nearest neighbors classifier using $\phi'$ to bayes-optimal over $\Ds^\phi$, where $\phi$ is the original feature map we are considering. In other words, this lemma implies that small perturbations to the feature map do not affect classification.

Next, we show that the entire support of the target distribution, $\Dt = (\mu_t, \eta_t)$, can be covered using the regions $\cX_{p, r}$. 

\begin{lemma}\label{lem:p_r_cover}
Let $\rho > 0$. Then there exists $p, r > 0$ such that the following holds. For all $\phi \in \Phi$ that realize SIRM ($\Ds, \Dt$) and for which $\Ds^\phi$ has margin at least $\rho$, $$\supp(\Ds), \supp(\Dt) \subseteq \cX_{p, r}^\phi.$$
\end{lemma}

\begin{proof}
Let $r = \rho\left(\frac{1}{2} - \frac{1}{\Lambda}\right)$. By the Definition of $\Lambda$, $r > 0$. Now let $x \in \supp(\Dt)$ be arbitrary.

Because $\phi$ contracts $(\Ds, \Dt)$, there exists $x' \in \supp(\mu_s)$ such that $d_\phi(x, x') < \frac{\rho^\phi}{\Lambda}$, where $\rho^\phi$ is the margin of $D_s^\phi$. It follows that
\begin{equation*}
\begin{split}
d_\phi(x, x') &< \frac{\rho^\phi}{\Lambda} \\
&= \frac{1}{2}\rho^\phi - \left(\frac{1}{2} - \frac{1}{\Lambda}\right)\rho^\phi \\
&\leq \frac{1}{2}\rho^\phi - \left(\frac{1}{2} - \frac{1}{\Lambda}\right)\rho \\
&= \frac{\rho^\phi}{2} - r.
\end{split}
\end{equation*}
Finally, by Lemma \ref{lem:compact_ABUSE}, there exists $s > 0$ such that for all $\phi \in \Phi$, $\phi(B(x, s)) \subseteq B(\phi(x), r))$. This implies that $B(x', s) \subseteq B_\phi(x', r)$ for all $x'$. Finally, we take $$p = \inf_{x \in \supp(\mu_s)}\mu_s(B(x, s).$$ It suffices to show that $p > 0$. To do so, observe that $\supp(\mu_s)$ is closed and therefore compact (as $\cX$ is compact by assumption). Take an open cover of $\mu_s$ by balls of radius $s/2$. Then it has a finite sub-cover. Each of htese balls have positive mass under $\mu_s$, and furthermore every ball $B(x, s)$ where $x \in \supp(\mu_s)$ must fully contain at least one of these balls. It follows that $\mu_s(B(x, s)) \geq q$, where $q>0$ is the minimum mass of one of these balls. Since $q > 0$, it follows that $p > 0$, as desired. 
\end{proof}

\begin{lemma}\label{lem:strong_k_nn_converge}
Let $\rho > 0$. Then there exists $N > 0$ such that for all $\phi$ that relate $(\Ds, \Dt)$ such that $\Ds^\phi$ has margin at least $\rho$, if $n \geq N$, then with probability at least $1- \frac{1}{n^2}$ over $S \sim \Ds^n$, $$R(\n_S^\phi, \Dt) < R_t^* + \frac{1}{n^2}.$$ 
\end{lemma}

\begin{proof}
Let $\phi$ relate $\Ds, \Dt$, and suppose $\Ds^\phi$ has margin $\rho^\phi$. let $r, p$ be as in Lemma \ref{lem:p_r_cover}. Because $\phi$ relates $\Ds, \Dt$, observe that for all $x \in \supp(\mu_t)$, $$g_{Dt}(x) = g_{\Ds^\phi}\left(\argmin_{z \in \supp(\Ds^\phi)}d_\cZ(z, \phi(x)\right).$$ To see this, observe that Definition \ref{defn:covers} implies that $min_{z \in \supp(\Ds^\phi)}d_\cZ(z, \phi(x)) < \frac{\rho^\phi}{2}$, and Definition \ref{defn:joint_classifies} implies that $z$ must be labeled by $g_{\Ds^\phi}$ the same as $x$ is by $g_{\Dt}$.

Next, select $N$ from Lemma \ref{lem:k_nn_p_delta_accurate}. It follows that since all $x \in \supp(\mu_t) \in \cX_{p, r}^\phi$, for $n \geq N$, we have that with probability at least $1- \frac{1}{n^4}$ over $S \sim \Ds^n$, $$\n_S^\phi(x)  = g_{\Ds^\phi}\left(\argmin_{z \in \supp(\Ds^\phi)}d_\cZ(z, \phi(x)\right)=g_{Dt}(x).$$ A standard of markov's inequality that converts the expected loss into a loss bound with high probability completes the proof.
\end{proof}

We are now prepared to prove Theorem \ref{thm:k_nn_converge}.

\begin{proof}
Any $\phi$ that relates $(\Ds, \Dt)$ has positive margin, and  so the previous lemma applies for sufficiently large $n$. Since $\frac{1}{n^2} \to 0$, it immediately follows that $\n_S^\phi$ converges in risk to the bayes optimal of $\Dt$, as desired. 
\end{proof}

\section{Proof of Theorem \ref{thm:setting_2_upper_bound}}\label{app:pure_source}

\subsection{Description of our learning rule}

We begin with our learning rule, $L$, that achieves the bound given in Theorem \ref{thm:setting_2_upper_bound}.

\begin{algorithm}[tbh]
   \caption{$\nLabel(S \sim \Ds^n)$}
   \label{alg:setting_2}
\begin{algorithmic}[1]
   \State $S_{tr} \gets  \{(x_i, y_i): 1 \leq  i \leq n/4\}$
   \State $S_{loss} \gets \{(x_i, y_i): n/4 < i \leq n/2 \}$
   \State $S_{margin} \gets \{ (x_i, y_i): n/2 < i \leq 3n/4 \}$
   \State $S_{final} \gets \{ (x_i, y_i): 3n/4 < i \leq n \}$
   \State $\epsilon \gets  n^{-1/3}$
   \vspace{1mm}
   \State $\Phi_\epsilon = \left\{\phi: \textsc{source\_loss}(\phi, S_{loss}) < \epsilon \right\}$
   \State $\hat{\phi} \gets \argmax_{\phi \in \Phi_\epsilon} \textsc{source\_margin}(\phi, S_{margin})$
   \State \textbf{return} $\n_{S_{final}}^{\hat{\phi}}$
\end{algorithmic}
\end{algorithm}


\subsection{Bounding the error in estimating the loss}

Our method for estimating the loss over the source distribution that a nearest neighbors classifier is given in Algorithm \ref{alg:setting_2}. We simply evaluate the empirical risk using nearest neighbors over the designated loss set, $S_{loss}$. 

We now bound the accuracy of this method using the following Lemma.

\begin{lemma}\label{lem:bound_source_accuracy}
Let $\cD$ be an arbitrary data distribution. Let $S_{tr}$ be a set of $n$ labeled points, and let $S_{loss} \sim \cD^n$ be an i.i.d sample that is independent of $S_{tr}$. Then there exists $N >0$ such that for all $n \geq N$, with probability at least $1 - \frac{1}{n^2}$ over $S_{tst} \sim \Ds^n$, for all $\phi \in \Phi$, $$|source\_loss(\phi, S_{tr}, S_{tst}) - R(\n_{S_{tr}}^\phi, \cD)| < n^{-1/3}.$$ 
\end{lemma}

\begin{proof}
Fix $\epsilon = n^{-1/3}$, and define $E$ as the event that the empirical risk induced by each $\phi \in \Phi$ is representative of the true risk. That is, $$E = \ind \left(\sup_{\phi \in \Phi} \left|R(\n_{S_{tr}}^\phi, \cD) - \frac{1}{n}\sum_{(x,y) \in S_{loss}}\ind\left(\n_{S_{tr}}^\phi(x) \neq y\right)\right| < \epsilon\right).$$

Our goal is to show that $E$ holds with probability at least $1 - \frac{1}{n^2}$, for sufficiently large $n$. The key observation is that for all $\phi \in \Phi$, $$\ind\left(\n_{S_{tr}}^\phi(x) \neq y\right) = 1 - h_{S, \phi}((x,y)),$$ where $h_{S, \phi}(x, y)$ is as defined in Definition \ref{defn:h_s_phi}. Thus, it follows that  
\begin{equation*}
\begin{split}
E &= \ind \left(\sup_{\phi \in \Phi} \left|R(\n_{S_{tr}}^\phi, \cD) - \frac{1}{m}\sum_{(x,y) \in S_{loss}}\ind\left(\n_{S_{tr}}^\phi(x) \neq y\right)\right| < \frac{\epsilon}{2}\right) \\
&= \ind\left( \sup_{h \in \cH(S_{tr}, \Phi)} \left|\mathbb{E}_{(x,y) \sim \cD}[1 - h_{S, \phi}(x,y)] - \frac{1}{m}\sum_{(x,y) \in S_{loss}}1 - h_{S, \phi}(x,y)\right| < \frac{\epsilon}{2}\right) \\
&= \ind\left( \sup_{h \in \cH(S_{tr}, \Phi)} \left|\mathbb{E}_{(x,y) \sim \cD}[h_{S, \phi}(x,y)] - \frac{1}{m}\sum_{(x,y) \in S_{loss}}h_{S, \phi}(x,y)\right| < \frac{\epsilon}{2}\right).
\end{split}
\end{equation*}

To analyze the latter quantity, a standard application of the fundamental theorem of statistical learning (see Shavel-Schwartz and Ben-David) implies that $E$ holds with probability $1 - \delta$ provided that $|S_{loss}| = n \geq \Omega\left(\frac{vc\left(\cH(S_{tr}, \Phi)\right) + \ln \frac{1}{\delta}}{\epsilon^2}\right)$. 

Fix $\delta = \frac{1}{n^2}$. By Lemma \ref{lem:vc_bound_classification}, $vc\left(\cH(S_{tr}, \Phi)\right) \leq c_1\partial(\Phi)\log\left(n + \partial(\Phi)\right)$. Substituting this, along with $\epsilon = n^{-1/3}$, we see that 
\begin{equation*}
\begin{split}
\frac{vc\left(\cH(S_{tr}, \Phi)\right) + \ln \frac{1}{\delta}}{\epsilon^2} &\leq \frac{c_1\partial(\Phi)\log\left(n + \partial(\Phi)\right) + 2 \ln n}{n^{-2/3}} \\
&\leq Cn^{2/3}\log n,
\end{split}
\end{equation*}
where $C$ is some constant that depends on $\partial(\Phi)$. Since $n$ assymptotically dominates this quantity, it follows that for sufficiently large $n$, we indeed have $n \geq \Omega\left(\frac{vc\left(\cH(S_{tr}, \Phi)\right) + \ln \frac{1}{\delta}}{\epsilon^2}\right)$, which proves the desired result.
\end{proof}

\begin{algorithm}[tb]
   \caption{$source\_loss(\phi, S_{tr}, S_{loss})$}
   \label{alg:source_loss}
\begin{algorithmic}[1]
	\State \textbf{return} $\frac{1}{|S_{loss}|}\sum_{(x,y) \in S_{loss}} \ind\left(\n^{\phi}_{S_{tr}}(x) \neq y\right)$ 
\end{algorithmic}
\end{algorithm}

\subsection{Bounding the error in estimating the margin}

Our method for estimating the margin of a distributions, $\Ds^\phi$, is given in Algorithm \ref{alg:source_margin}. The main idea is to split the set, $S_{source}$, into two equal parts, $S_{source}^a$, and $S_{source}^b$. We then use a nearest neighbors classifier over $S_{tr}$ to label the points in both $S_{source}^a$ and $S_{source}^b$. Finally, we measure the distance between differently labeled points from $S_{source}^a$ and $S_{source}^b$ respectively. For technical reasons, when comparing distances between $S_{source}^a$ and $S_{source}^b$, we only compare points that have the same index. This allows us to exploit independence between each comparison we make.

We now show that this method is likely to accurately estimate margins by showing that it gives good estimates for $\alpha^\phi$, which is described in Definition \ref{defn:distribution_margin}.  

\begin{lemma}\label{lem:source_margin_bound}
There exists $N$, such that for all $n \geq N$, if $S_{tr}$ is a set of $n$ labeled points, with probability at least $1-\frac{1}{n}$ over $S_{source} \sim \Ds^{n}$ and $S_{loss} \sim \Ds^n$, at least one of the two conditions will hold:
\begin{enumerate}
	\item $source\_loss(\phi, S_{tr}, S_{loss}) > R_s^* + O(n^{-1/4})$.
	\item $\Pr[\alpha^\phi < source\_margin(\phi, S_{tr}, S_{source})] < O(n^{-1/4})$. 
\end{enumerate}
\end{lemma}

\begin{proof}
For $\phi \in \Phi$ and $r \geq 0$, let $\cq_{\phi, r, S_{tr}}$ be as defined in Definition \ref{defn:q_phi_r_S}, so that $$\cq_{\phi, r, S_{tr}}(x, x') = \begin{cases}1 & d_\phi(x, x') < r\text{ and }\n_S^\phi(x) \neq \n_S^\phi(x') \\ 0 & \text{otherwise} \end{cases}.$$ Observe that $$source\_margin(\phi, S_{tr}, S_{source}) = \max \left\{r: \frac{1}{n} \sum_{i=1}^n \cq_{\phi, r, S_{tr}}(x_i^a, x_i^b) = 0\right\}.$$ This is because $\cq_{\phi, r, S_{tr}}(x_i^a, x_i^b) = 0$ if and only if either $x_i^a$ and $x_i^b$ are given the same labels, or if they have distance (under $\phi$) of at least $r$. 

Define $\alpha_{S_{tr}}^\phi$ as the random variable where for $x, x' \sim \mu_s$, $$\alpha_{S_{tr}}^\phi = \begin{cases}d_\phi(x, x') & \n_{S_{tr}}^\phi(x) \neq \n_{S_{tr}}^\phi(x') \\ \infty & \text{otherwise} \end{cases}.$$ The variable $\alpha_{S_{tr}}^\phi$ is closely related to $\alpha^\phi$, the only difference is that we replace the bayes optimal classifier, $g_{\cD_s}$ with $\n_{S_{tr}}^\phi$. 

To relate $\alpha_{S_{tr}}^\phi$ to our previous quantities, observe that $$\Pr[\alpha_{S_{tr}}^\phi \leq r] = \mathbb{E}_{(x, x') \sim \mu_s^2}[\cq_{\phi, r, S_{tr}}(x, x')].$$

Because the set of classifiers, $\cQ(\Phi, S_{tr}) = \{\cq_{\phi, r, S_{tr}}: \phi \in \Phi, r \geq 0\}$ has bounded VC-dimension, $c_3\partial(\Phi)\log(n + \partial(\Phi))$,  we can apply uniform convergence to see that $\Pr[\alpha_{S_{tr}}^\phi \leq r]$ must be close to its expectation with high probability over all $\phi, r$. More precisely, by applying the same argument as in the proof of Lemma \ref{lem:bound_source_accuracy}, a we have the for $n$ sufficiently large, with probability at least $1 - \frac{1}{n^2}$ over $S_{source} \sim \Ds^{n}$, for all $\cq_{\phi, r, S_{tr}} \in \cQ(\Phi, S_{tr})$, $$\left|\mathbb{E}_{(x, x') \sim \mu_s^2}[\cq_{\phi, r, S_{tr}}(x, x')] - \frac{1}{n} \sum_{i=1}^n \cq_{\phi, r, S_{tr}}(x_i^a, x_i^b)\right| < n^{-1/3}.$$ By substituting the definition of $\alpha_{S_{tr}}^\phi$ along with our observation about $source\_margin(\phi, S_{tr}, S_{source})$, it follows that $$\Pr[\alpha_{S_{tr}}^\phi < source\_margin(\phi, S_{tr}, S_{source})] < n^{-1/3}.$$ We now turn our attention to showing that $\alpha_{S_{tr}}^\phi$ must indeed serve as a reasonable approximation for $\alpha^\phi$. To do so, observe that if $\alpha_{S_{tr}}^\phi$ and $\alpha^\phi$ are constructed from the same random variables, $x, x' \sim \mu_s$, then they only differ if $\n_{S_{tr}}^\phi$ and $g_{\Ds}$ differ over either $x$ or $x'$. 

Suppose that $R(\n_{S_{tr}}^\phi , \Ds) = R(g_{\Ds}, \Ds) + \epsilon^\phi.$ Then it follows by Lemma \ref{lem:stuff_close_to_bayes_good} that $\Pr_{x \sim \mu_s}[g_{\Ds}(x) \neq \n_{S_{tr}}^\phi(x)] \geq \frac{\epsilon^\phi}{\Delta}$, where $\Delta$ is the label margin of $\Ds$. It follows by the rules of probability that the probability that $\alpha^\phi < r$ is at most $\frac{\epsilon^\phi}{\Delta}$ summed with the probability that $\alpha_{S_{tr}}^\phi < r$. That is, 
\begin{equation}\label{eqn:margin_nice}
\Pr[\alpha_{S_{tr}}^\phi < source\_margin(\phi, S_{tr}, S_{source})] < n^{-1/3} + \frac{\epsilon^\phi}{\Delta}
\end{equation} 
However, if $n$ is sufficiently large, then we have that with probability at least $1- \frac{1}{n^2}$ over $S_{loss} \sim \Ds^n$, for all $\phi$, $$\left|source\_loss(\phi, S_{tr}, S_{loss}) - \left(R(g_{\Ds}, \Ds) + \epsilon^\phi\right)\right| < O(n^{-1/3}).$$ This in turn implies that 
\begin{equation}\label{eqn:loss_nice}
source\_loss(\phi, S_{tr}, S_{loss}) > R(g_{\Ds}, \Ds) + \epsilon^\phi - O(n^{-1/3})
\end{equation}

By taking a union bound, it follows with probability at least $1 - \frac{2}{n^2}$, that the Equations \ref{eqn:loss_nice} and \ref{eqn:margin_nice} simulatenously hold over all $\phi$.

Finally, if $\epsilon^\phi \geq n^{-1/4}$, then for $n$ sufficiently large, condition number 2. from the statement of the Lemma must hold. Otherwise, if $\epsilon^\phi < n^{-1/4}$, condition 1. holds. Thus in either case, one of the two conditions hold which completes the proof.
\end{proof}

\begin{algorithm}[tb]
   \caption{$source\_margin(\phi, S_{tr}, S_{source})$}
   \label{alg:source_margin}
\begin{algorithmic}[1]
   \State $S_{source} = S_{source}^a \cup S_{source}^b$, $|S_{margin}^a| = |S_{margin}^b|$.
   \State $S_{source}^a = \{(x_1^a, y_1^a), \dots, (x_{n}^a, y_{n}^a)\}$.
   \State $X^a = \{x_1^a, \dots, x_{n}^a\}$.
   \State $S_{source}^b = \{(x_1^b, y_1^b), \dots, (x_{n}^b, y_{n}^b)\}$.
   \State $X^b = \{x_1^b, \dots, x_{n}^b\}$.
   \For{$i = 1 \dot n$}
   \State $d_i = d_\phi(x_i^a, x_i^b)$.
   \If{$\n_{S_{tr}}^\phi(x_i^a) = \n_{S_{tr}}^\phi(x_i^b)$}
   \State $d_i = \infty$
   \EndIf
   \EndFor
   \State \textbf{return} $\min_{1 \leq i \leq n}d_i$. 
\end{algorithmic}
\end{algorithm}

\subsection{Proving the theorem}

We first show a Lemma that implies that the feature map selected by our algorithm, $\hat{\phi}$, is likely to realize the SIRM assumption on $(\Ds, \Dt)$.

\begin{lemma}
Let $\phi^* \in \Phi$ be any SIRM realizing feature map, and suppose that $\Ds^{\phi^*}$ has margin $\rho^*$. Then for all $\delta > 0$, there exists $N$ such that if $n \geq N$, with probability at least $1-\delta$ over $S \sim \Ds^n$, $\hat{\phi}$ (defined Line 6 of \nLabel) is a SIRM realizing feature map for $(\Ds, \Dt)$, and has margin at least $\frac{\rho^*}{2}$. 
\end{lemma}

\begin{proof}
Assume towards a contradiction, that for $\delta > 0$, there exist arbitrarily large values of $n$ for which with probability at least $\delta$, $\hat{\phi}$ has margin less than $\frac{\rho^*}{2}$. 

For $n$ sufficiently large, with probability at least $1 - O(\frac{1}{n})$, applying Lemmas \ref{lem:bound_source_accuracy} and \ref{lem:source_margin_bound} we have that, for all $\phi \in \Phi$, $$|source\_loss(\phi, S_{tr}, S_{loss}) - R(\n_{S_{tr}}^\phi, \cD)| < n^{-1/3},$$ and that one of the two conditions hold as well:
\begin{enumerate}
	\item $source\_loss(\phi, S_{tr}, S_{loss}) > R_s^* + O(n^{-1/4})$.
	\item $\Pr[\alpha^\phi < source\_margin(\phi, S_{tr}, S_{source})] < O(n^{-1/4})$. 
\end{enumerate}
Because these equations hold for $\phi^*$, the smallest observed empirical loss must be at most $R_s^* + n^{-1/3}$. It follows that $\hat{\phi}$ must incur empirical loss at most $R_s^* + 2n^{-1/3}$ which implies that condition 2. must apply to $\hat{\phi}$. 

Furthermore, because $k_n > \log n$, we have that with high probability, $\n_{S_{tr}}^{\phi^*}$ will match the Bayes-optimal classifier, $g_{\Ds}$ for all points in $S_{source}$. It follows that the observed margin, $source\_margin(\phi^*, S_{tr}, S_{source})$ will be at least $\rho^*$. 

Combining all of this, we see that $$\Pr[\alpha^{\hat{\phi}} < \rho^*] < n^{-1/4}.$$ Now let $n_1, n_2, \dots$ be a sequence of integers going to infinity so that for each $n_i$, with probability at least $\delta$, $\hat{\phi}$ has a margin less than $\frac{\rho^*}{2}$. Because $\delta$ is fixed, it follows that for sufficiently large $n_i$, there exist $\hat{\phi}_i$ and $S_i$ such that all of the equations above hold.

Because $\Phi$ is compact, there exists an infinite subsequence of the $n_i$s for which $\hat{\phi}_i$ converges (using the distance metric over $\Phi$) to some $\phi$. Relabel our sequence so that without loss of generality, $\hat{\phi}_i \to \phi$. 

The key observation is that the variable, $\alpha^\phi$ is Lipschitz with respect to the distance metric over $\Phi$. In particular, if $|\phi - \phi'| < r$, then $\alpha^\phi - \alpha^{\phi'} < 2r$. 

Using this, observe that for sufficiently large values of $i$, we have that $d(\hat{\phi}_i, \phi) < \frac{\rho^*}{8}$. Substituing this, it follows that for all sufficiently large $i$,
$$\Pr[\alpha^\phi < \frac{3\rho^*}{4}] \leq \Pr[\alpha^{\hat{\phi}} < \rho^*] < n_i^{-1/4}.$$ Since $n_i$ can be arbitrarily large it follows that $\Pr[\alpha^\phi < \frac{3\rho^*}{4}] = 0$ which implies $\Ds^\phi$ must have margin at least $\frac{3\rho^*}{4}$.

However, this in term implies that for all sufficiently large $i$, $\hat{\phi}_i$ too must have margin at least $\frac{3\rho^*}{4} - \frac{\rho^*}{4} = \frac{\rho^*}{2}$. Here we are again exploiting the fact that the margin is Lipschitz. 

This finally gives us a contradiction, as we previous assumed that all $\hat{\phi}_i$ had margin less than $\frac{\rho^*}{2}$. 

\end{proof}

We are now prepared to prove Theorem \ref{thm:setting_2_upper_bound}.

\begin{proof}
Fix $\epsilon, \delta > 0$. The previous Lemma implies that for sufficiently large values of $n$, with probability $1-\frac{\delta}{2}$ we will select some $\hat{\phi}$ that has margin at least $\frac{\rho^*}{2}$. Lemma \ref{lem:strong_k_nn_converge} implies that for $n$ sufficiently large (in a way that only depends on $\rho^*, \Ds$), with probability at least $1- \frac{\delta}{2}$ over $S_{final} \sim \Ds^{n/4}$, $$R(\n_{S_{final}}^{\hat{\phi}}, \Dt) < R_t^* + \epsilon.$$ Crucially, $S_{final}$ is completely independent of $\hat{\phi}$, which is learned purely using $S_{tr}, S_{loss},$ and $S_{source}$. Taking a union bound implies the desired result.
\end{proof}

\section{Proof of Theorem \ref{theorem:type-1-lower-bound}}

\begin{proof}
Fix $\epsilon > 0$. Let $\phi_1$ relate $(\Ds, \Dt)$, and let $\phi_2 \in \mathcal{S}(\Phi) \setminus \Phi^*$ be a feature map that source-preserves $\Ds$ but fails to relate $(\Ds, \Dt)$. We will construct $\Ds'$ and $\Dt'$ using $\phi_1$ and $\phi_2$. 

Because $\phi_2$ fails to relate $(\Ds, \Dt)$, there exists $x \in \supp(\mu_t)$ such that $\phi_2(x) = z \notin \supp(\mu_s^{\phi_2})$. Note that if this doesn't hold, then we can simply use the construction from the proof of Theorem \ref{thm:lower_bound_setting_unlabeled}. The point $x$ will be central to constructing both $\Ds'$ and $\Dt'$. We begin by constructing $\Ds'$. 

\paragraph{Constructing $\Ds'$: } Let $\alpha > 0$ be a small value. Then by Assumptions \ref{assumption:infinite_inverse} and \ref{assumption:indomitable}, there exists $x_1, x_2 \in \mathcal{X}$ and $\phi_1^{\alpha}, \phi_2^\alpha \in \Phi$ such that the following conditions hold:
\begin{enumerate}
	\item $d(\phi_1, \phi_1^\alpha), d(\phi_2, \phi_2^\alpha) < \alpha$. 
	\item $\phi_1^\alpha(x_1) = \phi_1^\alpha(x) \neq \phi_1^\alpha(x_2)$. 
	\item $\phi_2^\alpha(x_1) \neq \phi_2^\alpha(x) \neq \phi_2^\alpha(x_2)$.
\end{enumerate}
Here, $\phi_1^\alpha$ and $\phi_2^\alpha$ are chosen using Assumption \ref{assumption:indomitable}, while the existence of $x_1$ and $x_2$ is based on Assumption \ref{assumption:infinite_inverse}. We also let $x_1', x_2'$ be two points such that $$0 < d_{\phi_1^\alpha}(x_1, x_1') << d_{\phi_1^\alpha}(x_1,x_2),\text{ and }0 < d_{\phi_2^\alpha}(x_2, x_2') << d_{\phi_2^\alpha}(x_1,x_2).$$ 

Next, let $\mu_s'$ be a measure over $\mathcal{X}$ obtained by the following steps.
\begin{enumerate}
	\item Begin with $\mu_s$, the measure of $\Ds$ over $\mathcal{X}$.
	\item \textit{Remove} all points in $\supp(\mu_s)$ that lie within a distance of $r$ from the set $\{x_1, x_1', x_2, x_2'\}$. 
	\item Pick $s$ such that any two points in $\{x_1, x_1', x_2, x_2'\}$ have distance larger than $4s$. Insert balls of probability mass $\epsilon/8$ centered at each of these points, so that $\mu_s'(B(x, s)) = \frac{\epsilon}{4}$ for $x \in \{x_1, x_1', x_2, x_2'\}$. 
\end{enumerate}
Observe that $\mu_s'$ is constructed from $\mu_s$ by adding a region of mass $\epsilon/2$ (and appropriately down-sizing all other regions). Furthermore, if $r$ is appropriately chosen, then the region being removed from $\mu_s$ can also be forced to have size at most $\epsilon/2$. It follows that $W(\mu_s, \mu_s') < \epsilon$. Next, we define the conditional distribution, $\eta_s'$ with the following steps. Let $y_1 \neq y_2$ be two labels in $\mathcal{Y}$.
\begin{enumerate}
	\item For $x \in B(x_1, s): \eta_s'(y_1|x) = 1$.
	\item For $x \in B(x_1', s): \eta_s'(y_2|x) = 1$.
	\item For $x \in B(x_2, s): \eta_s'(y_2|x) = 1$.
	\item For $x \in B(x_2', s): \eta_s'(y_1|x) = 1$.
	\item For all other $x$, $\eta_s'(y|x) = \eta(y|x)$. 
\end{enumerate}
Basically, we force the conditional distribution near $x_1$ and $x_2$ to be $y_1$ and $y_2$ respectively. For $x_1', x_2'$, this is reversed. This construction only modifies $\eta_s$ at points where $\mu_s$ is modified, and it follows that $W(\Ds, \Ds') < \epsilon$. 

Furthermore, observe that $\phi_1^\alpha$ and $\phi_2^\alpha$ both source-preserve $\Ds'$. This occurs because $r$ and $s$ are chosen to be small enough so that the 4 balls, $B(x_1, s), B(x_2, s), B(x_1', s), B(x_2', s)$ are all mapped to disjoint areas under both $\phi_1$ and $\phi_2$. 

\paragraph{Constructing $\Dt'$:} Next, we will construct $\Dt'$ by giving a choice of two possible target distribution, $\Dt^1$ and $\Dt^2$. We let $\mu_t'$ be a point mass that is concentrated at $x$. We let $\eta_t^1(y_1|x) = 1$ and $\eta_t^2(y_2|x) = 1$. This gives us $\Dt^1$ and $\Dt^2$.

Observe that $\Dt^1$ is SIRM related to $\Ds'$ by $\phi_1^\alpha$, and $\Dt^2$ is SIRM related to $\Ds'$ by $\phi_2^\alpha$. This is because $x$ is mapped to $x_1$ by $\phi_1^\alpha$, and the same holds respectively for $x_2$ and $\phi_2^\alpha$. 

\paragraph{Finishing the proof:} We now show that our learner will have a large error over either some choice of $\Dt' \in \{\Dt^1, \Dt^2\}$. To do so, suppose $\Dt'$ is randomly chosen from this set. It follows that our learning rules expected loss is:
\begin{equation*}
\begin{split}
\mathbb{E}_{\Dt' \sim \{\Dt^1, \Dt^2\}}\mathbb{E}_{S \sim (\Ds')^n} R(L(S), \Dt') &= \mathbb{E}_{S \sim (\Ds')^n} \mathbb{E}_{\Dt' \sim \{\Dt^1, \Dt^2\}} R(L(S), \Dt') \\
&= \mathbb{E}_{S \sim (\Ds')^n} \mathbb{E}_{i \sim \{1, 2\}} \ind(L(S)(x) \neq y_i) \\
&= \frac{1}{2}.
\end{split}
\end{equation*}

From here, the desired result follows by a straightforward application of markov's inequality.

\end{proof}

\section{Proof of Theorem \ref{thm:unlabel_upper_bound}}

\subsection{Description of the learning rule}

We give the learning rule that achieves the bound given in Theorem \ref{thm:unlabel_upper_bound}

\begin{algorithm}[tbh]
   \caption{$\ualg(S \sim \Ds^n, U \sim \mu_t^m)$}
   \label{algorithm:unlabeled}
\begin{algorithmic}[1]
    \State $S_{tr} \gets  \{(x_i, y_i): 1 \leq  i \leq n/5\}$
   \State $S_{loss} \gets \{(x_i, y_i): n/5 < i \leq 2n/5 \}$
   \State $S_{margin} \gets  \{(x_i, y_i): 2n/5 < i \leq 3n/5 \}$
   \State $S_{margin, t} \gets  \{(x_i, y_i): 3n/5 < i \leq 4n/5 \}$
   \State $S_{final} \gets \{(x_i, y_i): 4n/5 < i \leq n \}$
   \State $\epsilon \gets n^{-1/3}$ 
        \vspace{1mm}

   \State $\Phi_\epsilon = \left \{\phi: \textsc{source\_loss}(\phi, S_{tr}, S_{loss}) < \epsilon \right\}$
   \State $\rho_s(\phi) = \textsc{source\_margin}(\phi, S_{tr}, S_{margin})$ 
   \State $\rho_t(\phi) = \textsc{target\_margin}(\phi, S_{margin, t}, U)$
          \vspace{1mm}

   \State $\hat{\phi} \gets \argmax_{\phi \in \Phi_\epsilon} \rho_s(\phi) - \Lambda\rho_t(\phi)$

   \State \textbf{return} $\n_{S_{tr}}^{\hat{\phi}}$
\end{algorithmic}
\end{algorithm}

\subsection{Analyzing the procedure, $target\_margin$}

We begin by describing the process used to estimate how far data from $\Dt$ is from data from $\Ds$ under a feature map, $\phi$. The subroutine is given in Algorithm \ref{alg:target_margin}, where $S$ is a labeled set of points drawn from $\Ds$, and $U$ is an unlabeled set of points drawn from $\mu_t$, the marginal $\cX$-distribution of $\Dt$. 

\begin{algorithm}[tb]
   \caption{$target\_margin(\phi, S_{margin, t}, U)$}
   \label{alg:target_margin}
\begin{algorithmic}[1]
   \State $U = \{u_1, \dots, u_m\}$.
   \State $S_{margin, t} = \{(x_1, y_1), \dots, (x_n, y_n)\}$.
   \State $l = \min(m, [\sqrt{n}])$.
   \State $U' = \{u_1, \dots, u_l\}$.
   \For{$1 \leq i \leq l$}
        \State $X_i = \{(x_{il - l + 1}, \dots, x_{il}\}$
      \EndFor
   \State \textbf{return} $\max_{1 \leq i \leq l} \min_{x \in X_i} d_\phi(u_i, x)$. 
\end{algorithmic}
\end{algorithm}

The idea is for each $u \in U$, we assign it its own set of $l$ points sampled from $\mu_s$. Then, we take the distance from $u$ to the closest point in its assigned set. Finally, taking the max of all of these gives us an approxmation of the furthest distance any $u \in \supp(\mu_t)$ has from $\supp(\mu_s)$ when using the distance metric, $d_\phi$. 

We now show that this procedure approximates the quantity, $\beta_l^\phi$, which is defined in Definition \ref{defn:beta_distribution_margin_n}.

\begin{lemma}\label{lem:target_margin_bound}
There exists $N > 0$ such that if $n> N$, and $m \geq \sqrt{n}$, then with probability at least $1 - \frac{1}{n}$ over $S_{margin, t} \sim \Ds^n$, and $U \sim \Dt^m$, for all $\phi \in \Phi$, $$\Pr[\beta_{l_n}^\phi \geq target\_margin(\phi, S_{margin, t}, U)] \leq n^{-1/6},$$  where $l_n$ is the largest integer at most $\sqrt{n}$. 
\end{lemma}

\begin{proof}
For convenience, let $l$ denote $l_n$. For $\phi \in \Phi$ and $r \geq 0$, let $\cq_{\phi, r, l}$ be as defined in Definition \ref{defn:upgraded_q_stuff}, so that $$\cq_{\phi, r, l}(u, x_1, \dots, x_l) = \begin{cases}1 & \exists_{1 \leq i \leq l} d_\phi(u, x_i) < r \\ 0 & \text{otherwise} \end{cases}.$$ Furthermore, let us relabel $X_i$ (defined in line 6 of Algorithm \ref{alg:target_margin}) so that $X_i = (x_1^i, \dots, x_l^i)$. It follows that  $$target\_margin(\phi, S_{margin, t}, U) = \inf \left\{r: \frac{1}{l} \sum_{i=1}^l \cq_{\phi, r, l}(u_i, x_1^i, \dots, x_l^i) = 1\right\}.$$ This is because $\cq_{\phi, r, l}(u_i, x_1^i, \dots, x_l^i) = 1$ if and only if some $x_j^i$ has distance less than $r$ from $u_i$.

Next, we relate this quantity to $\beta_l^\phi$ as follows. Observe that $$\Pr[\beta_{l_n} \geq r ] = \mathbb{E}_{u \sim \mu_t, x_1, \dots, x_l \sim \mu_s^l}[1 -\cq_{\phi, r, l}(u, x_1, \dots, x_l)].$$

Because the set of classifiers, $\cQ_l(\Phi) = \{\cq_{\phi, r, l}: \phi \in \Phi, r \geq 0\}$ has bounded VC-dimension, $c_4\partial(\Phi)\log(l + \partial(\Phi))$,  we can apply uniform convergence to see that $\Pr[\beta_l < r]$ must be close to its expectation with high probability over all $\phi, r$. More precisely, by applying the same argument as in the proof of Lemma \ref{lem:bound_source_accuracy}, a we have the for $n$ sufficiently large, with probability at least $1 - \frac{1}{l^2}$ over $S_{margin, t} \sim \Ds^{l}$ $U \sim \Dt^l$, for all $\cq_{\phi, r, l} \in \cQ_l(\Phi)$, $$\left|\mathbb{E}_{u \sim \mu_t, x_1, \dots, x_l \sim \mu_s^l}[1-\cq_{\phi, r, l}u, x_1, \dots, x_l)] - \frac{1}{l} \sum_{i=1}^l 1-\cq_{\phi, r, S_{tr}}(u_i, x_1^i, \dots, x_l^i)\right| < l^{-1/3}.$$ By substituting the definition of $\beta_n^\phi$ along with our observation about $target\_margin(\phi, S_{margin, t}, U)$, it follows that for all $r > target\_margin(\phi, S_{margin, t}, U)$, $$\Pr[\alpha_{S_{tr}}^\phi \geq r] < l^{-1/3}.$$ By the rules of probability, and the definition of an infimum, it follows that $$\Pr[\alpha_{S_{tr}}^\phi \geq target\_margin(\phi, S_{margin, t}, U) ] \leq l^{-1/3}.$$ Substituting the value of $\l$ gives the desired result.
\end{proof}

\subsection{Bounding the performance a given feature map, $\phi^* \in \Phi$}

We now consider a fixed feature map $\phi^*$ that realizes the SIRM assumption on $(\Ds, \Dt)$. The idea behind doing so is that this allows us to give a baseline over how $source\_margin$ and $target\_margin$ should be expected to behave. 

\begin{lemma}\label{lem:phi_star_performance_target}
Let $\phi^*$ realize the SIRM assumption on $(\Ds, \Dt)$. Suppose $\Ds^{\phi^*}$ has margin $\rho^* > 0$, and that $$\max_{x_t \in \supp(\mu_t)}\min_{x_s \in \supp(\mu_s)}d_{\phi^*}(x_t, x_as) = \beta^*.$$ Finally, let $\rho^* - \Lambda \beta^* = \gamma^*$ where $\gamma^*>0$ by the fact that $\phi^*$ contracts $(\Ds, \Dt)$. Then for all $\delta > 0$, there exists $N$ such that if $n \geq N$, $m \geq \sqrt{n}$, with probability at least $1-\delta$ over $S \sim \Ds^n$, $U \sim \mu_t^m$, the following three things hold:
\begin{enumerate}
	\item $source\_loss(\phi^*, S_{tr}, S_{loss}) < R_s^* + n^{-1/3}$.
	\item $source\_margin(\phi^*, S_{tr}, S_{margin}) \geq \rho^*$.
	\item $target\_margin(\phi^*, S_{margin, t}, U) \leq \beta^* + \frac{\gamma^*}{2}$. 
\end{enumerate}
\end{lemma}

\begin{proof}
We bound the probability of each of these three things occuring separately, and then apply a union bound. 

First of all, for $n$ sufficiently large, the first claim holds with probability at least $1- O(\frac{1}{n^2})$. This follows directly from a combination of Lemma \ref{lem:bound_source_accuracy} along with Lemma \ref{lem:strong_k_nn_converge} being applied to $\Ds$ (as $\phi$ technically SIRM realizes on $(\Ds, \Ds)$). In particular, we have that with probability at least $1- \frac{1}{n^2}$ that 
\begin{equation}\label{eqn:we_are_so_accuarte}
R(\n_{S_{tr}}^{\phi^*}, \Ds) < R_s^* + \frac{1}{n^2}.
\end{equation}

Second, observe that Lemma \ref{lem:stuff_close_to_bayes_good} implies that the probability that $\n_{S_{tr}}^{\phi^*}$ differs from the Bayes-optimal is at most $\frac{1}{n^2\Delta}$, where $\Delta$ is the label margin of $\Ds$. It follows that with probability at least $1 - O(\frac{1}{n})$, $S_{tr}^{\phi^*}$ correctly labels all the points in $S_{margin}^a$ and $S_{margin}^b$ (see Algorithm \ref{alg:source_margin}). Thus, it follows that $source\_margin(\phi^*, S_{tr}, S_{margin}) \geq \rho^*$, as with correct labels it is impossible to observe two differently labeled poitns that are closer than $\rho^*$. 

Finally, by Lemma \ref{lem:compact_ABUSE}, there exists $\tau^* > 0$ such that $$\phi^*(B(x, \tau^*)) \subseteq B\left(\phi^*(x), \frac{\gamma^*}{2}\right).$$ Let $p = \min_{x \in \supp(\mu_s)} B(x, \frac{\tau^*}{2}).$ The argument in the proof of Lemma \ref{lem:compact_ABUSE} implies that $p > 0$. Thus, for $n$ sufficiently large, it follows that with probability at least $1 - (1-p)^l \geq 1 - e^{-pl}$ over $x_1, \dots, x_l \sim \mu_s^l$ that there will exist some $x_i \in B(x, \frac{\tau^*}{2})$. 

Since $\supp(\mu_s)$ is compact, it follows that we can take a finite covering of $\supp(\mu_s)$ with balls of radius $\frac{\tau^*}{2}$. If there are $C$ such balls, with probability at least $1 - (1-p)^l \geq 1 - Ce^{-pl}$ over $x_1, \dots, x_l \sim \mu_s^l$ that there will exist some $x_i \in B(x, \frac{\tau^*}{2})$ for all balls $B(x, \frac{\tau^*}{2})$ in our covering. Observe that this implies that for all $x \in \supp(\mu_s)$, there will exist some $x_i \in B(x, \tau^*)$. 

We now use this to show that the third condition is likely to hold. Pick $n$ sufficiently large so that $Ce^{-pl} < \frac{1}{n^2}$. It follows that with probability at least $1 - O(\frac{1}{n})$ over $S_{margin, t} \sim \Ds^{n/5}$ and $U \sim \Dt^m$, for all $1 \leq i \leq l$, for all $x \in \supp(\mu_s)$, there exists some $x_j^i \in B(x, \tau^*)$.

Suppose that this holds. For each $u_i \in U'$ (Algorithm \ref{alg:target_margin}), let $x_i^* \in \supp(\mu_s)$ be the point for which $d_{\phi^*}(u_i, x_i^*)$ is minimized. Thus $d_{\phi^*}(u_i, x_i^*) \leq \beta^*$ by the definition of $\beta^*$. However, by our claim above, we see that some $x_j^i$ must be in $B(x_i^*, 
tau^*)$ which implies that $\phi^*(x_j^i) \in B(\phi^*(x_i^*), \frac{\gamma^*}{2})$. Thus $d_{\phi^*}(x_j^i, u_i) \leq \beta^* + \frac{\gamma^*}{2}$.

Since this occurs for all $1 \leq i \leq l$, it follows that the maximum distance we observe is at most $\beta^* + \frac{\gamma^*}{2}$, which means $target\_margin(\phi, S_{margin, t}, U) \leq \beta^* + \frac{\gamma^*}{2}$, as desired. 

Since our three events all occur with probability $1 - O(1/n)$, it follows that if $n$ is sufficiently large, they simultaneously occur with probability at least $1-\delta$. This completes the proof. 

\end{proof}

\subsection{Proving the Theorem}

We are now prepared to prove Theorem \ref{thm:unlabel_upper_bound}. We start with the following Lemma.

\begin{lemma}
Let $\phi^*$ be as defined in Lemma \ref{lem:phi_star_performance_target}. Then for all $\delta > 0$, there exists $N > 0$ such that for all $n \geq N$, $m \geq \sqrt{n}$, with probability at least $1 - \delta$ over $S \sim \Ds^n$, $U \sim \Dt^m$, the outputted feature map $\hat{\phi}$ satisfies the following: let $\hat{\rho}$ denote the margin of $\Ds^{\hat{\phi}}$ and let $\hat{\beta}$ denote $$\hat{\beta} = \max_{x_t \in \supp(\mu_t)}\min_{x_s \in \supp(\mu_s)}d_{\hat{\phi}}(x_t, x_as).$$ Then $$\hat{\rho} - \Lambda \hat{\beta} \geq \frac{\gamma^*}{4}.$$
\end{lemma}

\begin{proof}
Assume towards a contradiction that this fails to occur and fix $\delta>0$ for that ails. For $n$ sufficiently large, with probability at least $1- \frac{\delta}{2}$, the premises of Lemmas \ref{lem:bound_source_accuracy}, \ref{lem:source_margin_bound},  \ref{lem:target_margin_bound}, and \ref{lem:phi_star_performance_target} are all simulatneously met. In particular, Lemmas \ref{lem:bound_source_accuracy}, \ref{lem:source_margin_bound} imply that $$|source\_loss(\phi, S_{tr}, S_{loss}) - R(\n_{S_{tr}}^\phi, \cD)| < O(n^{-1/3}),$$ and that one of the two conditions hold as well:
\begin{enumerate}
	\item $source\_loss(\phi, S_{tr}, S_{loss}) > R_s^* + n^{-1/4}$.
	\item $\Pr[\alpha^\phi < source\_margin(\phi, S_{tr}, S_{margin})] < n^{-1/4}$. 
\end{enumerate}
However, Lemma \ref{lem:phi_star_performance_target} implies that $$source
\_loss(\phi^*, S_{tr}, S_{loss}) < R_s^* + n^{-1/3}.$$ Since $\hat{\phi}$ has source loss at most $n^{-1/3}$ more than the optimal, it follows that condition 2 must hold. Thus, by additionally adding Lemma \ref{lem:target_margin_bound}, we see that $\hat{\phi}$ has the following properties:
\begin{enumerate}
	\item $\Pr[\alpha^{\hat{\phi}} < source\_margin(\hat{\phi}, S_{tr}, S_{margin})] < n^{-1/4}$
	\item $\Pr[\beta_{l_n}^{\hat{\phi}} \geq target\_margin(\hat{\phi}, S_{margin, t}, U)] \leq n^{-1/6}$.
\end{enumerate}
However, recall that $\hat{\phi}$ must maximize the quantity $source\_margin(\phi, S_{tr}, S_{margin}) - \Lambda target\_margin(\phi, S_{margin, t}, U)$. Since this quantity is at least $\frac{\gamma^*}{2}$ for $\phi^*$ (by Lemma \ref{lem:phi_star_performance_target}), it follows that
\begin{equation}\label{eqn:target_thing_i_want}
\Pr[\alpha^{\hat{\phi}} - \Lambda \beta_{l_n}^{\hat{\phi}} < \frac{\gamma^*}{2}] < O(n^{-1/6}).
\end{equation} 
In particular, this equation holds with probability at least $1- \frac{\delta}{2}$. However, by our assumption for arbitrarily large values of $n$, $\hat{\phi}$ fails to have the desired properties with probability $\delta$. 

Thus with probability at least $\frac{\delta}{2}$, for arbitrarily large values of $n_i$, there exists $\hat{\phi}_{n_i}$ such that Equation \ref{eqn:target_thing_i_want} holds but for which the desired property fails.

Let $n_1, n_2, \dots $ be any subsueqnce of integers so that the corresponding feature maps, $\hat{\phi}_{n_i}$ converge to some $\phi \in \Phi$. Note that this exists because $\Phi$ is compact. 

The key observation is that because $\alpha^\phi$ and $\beta_{l_n}^\phi$ are both Lipschitz with respect to $\phi$ (clearly small changes in the feature map cannot change these variables much), it follows that for all $\gamma > 0$, there exists $j$ such that for all $i > j$,

\begin{equation}\label{eqn:target_thing_i_want_last}
\Pr[\alpha^{\phi} - \Lambda \beta_{l_{n_i}}^{\phi} < \frac{\gamma^*}{2} - \gamma] < O(n_i^{-1/6}).
\end{equation} 
Howver, since $\beta_{t}^\phi \to \beta^\phi$ in distribution (Lemma \ref{lemma:beta_n_converges}), it follows that for $i$ sufficiently large, 
\begin{equation}\label{eqn:target_thing_i_want_last_last}
\Pr[\alpha^{\phi} - \Lambda \beta^{\phi} < \frac{\gamma^*}{2} - 2\gamma] < O(n_i^{-1/6}).
\end{equation} 

Since $n_i$ is arbitrarily large, observe that this implies that $$(\alpha^\phi) - \Lambda \max(\beta^\phi) \geq \frac{\gamma^*}{2} - 2\gamma.$$ Finally, since $\hat{\phi}$ gets arbitrarily close to $\phi$, it follows that $$(\alpha^{\hat{\phi}}) - \Lambda \max(\beta^{\hat{\phi}}) \geq \frac{\gamma^*}{2} - 3\gamma.$$  Taking $\gamma = \frac{\gamma^*}{12}$, and noting the definitions of $
\alpha^{\hat{\phi}}$ and $\beta^{\hat{\phi}}$, it follows that $\hat{\phi}$ precisely fullfills the conditions given in the statement of the lemma, and this completes the proof.

\end{proof}

We now prove Theorem \ref{thm:unlabel_upper_bound}.

\begin{proof}
Fix $\epsilon, \delta > 0$. The previous Lemma implies that for sufficiently large values of $n$, with probability $1-\frac{\delta}{2}$ we will select some $\hat{\phi}$ that has margin at least $\frac{\gamma^*}{4}$. We now use an argument identical to the argument given for Theorem \ref{thm:setting_2_upper_bound} and conclude the proof.
\end{proof}

\section{Proof of Theorem \ref{thm:lower_bound_setting_unlabeled}}

\begin{proof}
For $\phi \in \Phi$ such that $\phi$ source-preserves $\Ds$, define $g^\phi$ as the classifier over $\cX$ defined by $$g^\phi(x) = g_{\Ds^\phi}\left(\argmin_{z \in \supp(\Ds^\phi)}d_\cZ(z, \phi(x))\right),$$ with ties being broken arbitrarily.

Let $\phi_1 \in (\mathcal{S}(\Phi) \cap \Phi_{con} ) \setminus \Phi_{relates}$ and let $\phi_2 \in \mathcal{S}(\Phi)$. Because $\phi_1$ and $\phi_2$ both contract $(\Ds, \Dt)$, it follows that $g^{\phi_1}$ and $g^{\phi_2}$ are precisely well-defined over $\supp(\mu_t)$. However, because $\phi_1 \notin \Phi^*$, this implies that there exists $\epsilon > 0$ such that $$\mu_t\{x: g^{\phi_1}(x) \neq g^{\phi_2}(x)\} = \epsilon.$$ This holds from the fact that $g^{\phi_2}$ must match the Bayes-optimal, $g_{\Dt}$, whereas $g^{\phi_1}$ must fail to (otherwise $\phi_1$ would indeed SIRM realize on $(\Ds, \Dt)$). 


Define $$\eta_t^i(y|x) = \begin{cases} 1 & g^{\phi^i}(x) = y \\ 0 \text{ otherwise} \end{cases}.$$ Essentially this is a noiseless distribution that is purely classified by $g^{\phi^i}$. Let $\Dt^1 = (\mu_t, \eta_t^1)$ and $\Dt^2 = (\mu_t, \eta_t^2)$. The key observation is that if we randomly select $\Dt' \sim \{\Dt^1, \Dt^2\}$ and then apply our learning algorithm to $(\Ds, \Dt')$, then our learner must incur expected risk at least $\frac{\epsilon}{2}$. This is because whatever it outputs, it has a 50-50 chance of misclassifyign instances from in which $g^{\phi_1}$ and $g^{\phi_2}$ disagree. Thus our learner has expected risk at least $\frac{\epsilon}{2}$. Since $\epsilon > 0$ is fixed, this implies the desired result. 
\end{proof}

\section{Bounds on the Distance Dimension and Transfer Learning Guarantees}

\subsection{Proof of Theorem \ref{thm:bound_linear_dd}}

\begin{proof}
(Theorem \ref{thm:bound_linear_dd}) We will construct two maps, $\alpha: \Lin_{D, D} \to \R^{D^2}$, and $\beta: (\R^D)^4 \to \R^{D^2}$ such that for any $\phi \in \Lin_{D, D}$ and $x_1, x_2, x_3, x_4 \in (\R^D)$, $$\dist\phi(x_1, x_2, x_3, x_4) = sgn\left(\langle \alpha(\phi), \beta(x_1, x_2, x_3, x_4 \rangle\right).$$ This will immediately imply the result as it is well known that linear classifiers over $\R^n$ have vc dimension $n$. 

Letting $A_\phi$ be the $D \times D$ matrix associated with $\phi$, we have
\begin{equation*}
\begin{split}
dist\phi(x_1, x_2, x_3, x_4) &= sgn\left( d(\phi(x_1), \phi(x_2))^2 - d(\phi(x_3), \phi(x_4))^2 \right) \\
&= sgn \left( (x_1 - x_2)^tA_\phi^tA_\phi(x_1 - x_2) - (x_3 - x_4)^tA_\phi^tA_\phi(x_3 - x_4) \right) \\
&= sgn \left( \langle A_\phi^t A_\phi, (x_1 - x_2)(x_1- x_2)^t \rangle - \langle A_\phi^t A_\phi, (x_3 - x_4)(x_3- x_4)^t \rangle \right)\\
&= sgn \left(\langle A_\phi^t A_\phi, (x_1 - x_2)(x_1- x_2)^t - (x_3 - x_4)(x_3- x_4)^t \rangle \right)
\end{split}
\end{equation*}
Thus, letting $\alpha(\phi) = A_\phi^tA_\phi$ (cast as a vector in $\R^{D^2}$) and $\beta(x_1, x_2, x_3, x_4) = (x_1 - x_2)(x_1- x_2)^t - (x_3 - x_4)(x_3- x_4)^t$ suffices. 
\end{proof}

\subsection{Proof of Theorem \ref{thm:setting_1_upper_bound}}

\begin{proof}
(Theorem \ref{thm:setting_1_upper_bound}) Suppose $\phi^* \in \Phi$ realizes the Statistical IRM assumption for $\Ds, \Dt$. Define 
\begin{equation*}
E_1 = \ind\left(R(\n_S^{\phi^*}, \Dt) - R^*_t < \frac{\epsilon}{2} \right),
\end{equation*}
and 
\begin{equation*}
E_2 = \ind \left(\sup_{\phi \in \Phi} \left|R(\n^\phi_{S}, \Dt) - \frac{1}{m}\sum_{(x,y) \in T}\ind\left(\n^\phi_{S}(x) \neq y\right)\right| < \frac{\epsilon}{2}\right).
\end{equation*}
$E_1$ is thus the event that the source sample $S \sim \Ds^n$ gives rise to a $k_n$-nearest neighbors classifier which has small excess risk on $\Dt$ when composed with the realizing projection $\phi^*$. $E_2$ is the event that the empirical risks of $k_n$-nearest neighbor classifiers composed with feature maps $\phi \in \Phi$ are uniformly representative of their true risks on the target $\Dt$.

Our goal is to show that $E_1$ and $E_2$ jointly hold with probability at least $1-\delta$, as this would imply that our learned classifier has risk at most $R^*_t + \epsilon$, as desired. By Lemma \ref{lem:strong_k_nn_converge}, $E_1$ holds with probability at least $1- \frac{\delta}{2}$, so it suffices to show that $E_2$ holds with probability at least $1 - \frac{\delta}{2}$ as well. 

Fix any set of $n$ points, $\hat{S}$. It suffices to show that $\Pr_{T \sim \Dt^m}[E_2 = 1 | S = \hat{S}] \geq 1-\delta$, as integrating over all possibilities of $S$ would give the desired result. Consider the hypothesis class, $\cH_{\hat{S}} = \{h_\phi: \phi \in \Phi\}$ where $h_{\phi}: \cX \times \cY \to \{0, 1\}$ is defined as
\begin{equation*}
h_{\phi}(x,y) = \ind\left(\n^{\phi}_{\hat{S}}(x) \neq y \right).
\end{equation*} 
Observe that $h \in \cH_{\hat{S}}$ is a binary classifier over its domain. It follows that given $S = \hat{S}$,
\begin{equation*}
\begin{split}
E_2 &=\ind \left(\sup_{\phi \in \Phi} \left|R(\n^\phi_{S}, \Dt) - \frac{1}{m}\sum_{(x,y) \in T}\ind\left(\n^\phi_{S}(x) \neq y\right)\right| < \frac{\epsilon}{2}\right) \\
&= \ind\left( \sup_{h_{\phi} \in \cH_{\hat{S}}} \left|\mathbb{E}_{(x,y) \sim \Dt}[h_{\phi}(x,y)] - \frac{1}{m}\sum_{(x,y) \in T}h_{\phi}(x,y)\right| < \frac{\epsilon}{2}\right).
\end{split}
\end{equation*}

To analyze the latter quantity, it suffices to show that $vc(\cH_{\hat{S}}) \leq O\left(\partial(\Phi)\log \left(n+\partial(\Phi)\right)\right)$, as standard application of the fundamental theorem of statistical learning \cite{SS2014} implies $E_2$ holds with probability $1 - \frac{\delta}{2}$ provided that $m \geq \Omega\left(\frac{vc(\cH_{\hat{S}}) + \ln \frac{1}{\delta}}{\epsilon^2}\right)$.

To this end, suppose $\cH_{\hat{S}}$ shatters a set of $v$ points $V \subset \cX \times \cY$. Let $V=\{(x_1, y_1), \dots, (x_v, y_v)\}$, and let $\hat{S} = \{(x_1', y_1'), \dots, (x_n', y_n')\}$. The key observation is that for any $h_\phi \in \cH_{\hat{S}}$, the way $h_\phi$ labels a given point $(x, y) \in V $ is determined by the $k_n$-nearest neighbors of $\phi(x)$ in $\{\phi(x'_1), \dots, \phi(x'_n)\}$. Furthermore, these labels are fully determined by the set of all $\binom{n}{2}$ comparisons,
\begin{equation*}
\cC_{\phi, x} = \left\{\ind \left(d(\phi(x),\phi(x'_i)) \geq d(\phi(x),\phi(x'_j)\right): 1 \leq i < j \leq n \right\}.
\end{equation*}
Note that $\cC_{\phi, x}$ is a set of induced distance comparers (Definition \ref{defn:dist_comparer}). It follows that the number of distinct ways that $\cH_{\hat{S}}$ can label $V$ is at most the number of ways $\dist\Phi$ can label all $v\binom{n}{2}$ possible comparisons arising from $x_i \in V$ and $x'_j, x'_k\in \hat{S}$ with $j < k$. Thus, by the definition of the distance dimension $\partial(\Phi)$ and Sauer's Lemma, the number of ways $\cH_{\hat{S}}$ can label $V$ is at most 
\begin{equation*}
\left(e v\binom{n}{2}\right)^{\partial(\Phi)}
\end{equation*}
At the same time, because $\cH_{\hat{S}}$ shatters $V$, there exist precisely $2^v$ such labelings. Thus, we have 
\begin{equation*}
v \leq \partial(\Phi) \log \left(e v\binom{n}{2}\right).
\end{equation*}
 From here, straightforward algebra implies that $v = O\left(\partial(\Phi) \log \left(n + \partial(\Phi)\right)\right)$, as desired.
 
\end{proof}

\end{document}